\documentclass{article}

\PassOptionsToPackage{numbers,sort&compress}{natbib}

\usepackage[main, final]{neurips_2026}

\usepackage[utf8]{inputenc} 
\usepackage[T1]{fontenc}    
\usepackage{hyperref}       
\usepackage{url}            
\usepackage{booktabs}       
\usepackage{amsfonts}       
\usepackage{nicefrac}       
\usepackage{microtype}      
\usepackage{xcolor}         

\newlength{\leftpanel}
\newlength{\rightpanel}
\title{Blocked Gibbs meets Diffusion Transformers: Unsupervised Learning for Constraint Optimization}

\author{%
  Yudong W. Xu$^{1,2}$ \quad Wenhao Li$^{1,2}$ \quad Xiaoyu Wang$^1$ \quad Scott Sanner$^{1,2}$ \quad Elias B. Khalil$^1$
  \\
$^1$ University of Toronto \quad $^2$ Vector Institute \\
\texttt{\{wil.xu, chriswenhao.li, cnxiaoyu.wang\}@mail.utoronto.ca} \\
\texttt{ssanner@mie.utoronto.ca}, \texttt{elias.khalil@utoronto.ca}\\
}

\usepackage{amsmath}
\usepackage{amssymb}
\usepackage{mathtools}
\usepackage{amsthm}
\usepackage[capitalize,noabbrev]{cleveref}
\usepackage{multirow}
\theoremstyle{plain}
\newtheorem{theorem}{Theorem}[section]
\newtheorem{proposition}[theorem]{Proposition}

\usepackage[textsize=tiny]{todonotes}
\usepackage{subcaption}
\usepackage{wrapfig}
\usepackage{xcolor}

\DeclareMathOperator{\E}{\mathbb{E}}
\DeclareMathOperator{\KL}{KL}

\newcommand{\methodname}{BloGDiT}

\begin{document}

\maketitle

\begin{abstract}

Diffusion models have shown promise in learning to solve constraint optimization problems. However, they are mostly restricted to problems with binary variables and rely on graph neural networks, hindering their application to a broader range of problems such as those with general discrete variables or constraint structures that necessitate global rather than local reasoning. We investigate the use of Diffusion Transformers to address the aforementioned limitations. A naive implementation performs poorly due to a fundamental mismatch between the standard diffusion process and constraint solving: while the former applies \textit{small, incremental denoising} across \textit{all} variables, the latter requires \textit{substantially altering specific subsets of variables} to attain feasibility or optimality. Our method, \textit{\textbf{Blo}cked \textbf{G}ibbs \textbf{Di}ffusion \textbf{T}ransformer} (\methodname{}), is the first to address this limitation by replacing standard joint Gaussian denoising with blocked Gaussian denoising. 
\methodname{} uses iterative block resampling and anneals the block size over time to facilitate large, targeted edits within a block of variables. 
Across Sudoku, Graph Coloring, Maximum Independent Set, and MaxCut, \methodname{} matches or outperforms existing methods, demonstrating that blocked Gibbs-style diffusion provides a highly effective inductive bias for Transformer-based constraint satisfaction and optimization.


\end{abstract}

\section{Introduction}
\label{sec:intro}

Constraint optimization is fundamental to many real-world applications. Such problems are often solved by iterative methods: a candidate assignment is proposed, evaluated, and then modified. This iteration repeats until a high-quality solution is found, subject to a computation budget. Classical sampling and local search methods can be highly effective, but they typically solve each instance from scratch and rely on hand-designed (assignment) proposal mechanisms. This raises a natural question: \textit{can we learn a high-performance, reusable proposal mechanism whose cost is amortized across problem instances?}
Diffusion models~\cite{sohl2015deep,ho2020denoising,austin2021structured} provide an appealing answer. Their iterative denoising process can be interpreted as a learned refinement procedure: starting from a noisy assignment, the model repeatedly moves the state toward lower-energy regions. As illustrated in \Cref{fig:traj_intuition}, the dynamics of the learned algorithm stand in sharp contrast with pure discrete sampling-based search, which may explore excessively before reaching a good solution.

Recent works have applied diffusion models to combinatorial optimization and discrete reasoning~\cite{austin2021structured,ye2025beyond} but remain limited in scope. Most focus on binary graph problems and rely heavily on graph neural networks (GNNs) as the backbone architecture~\cite{sun2023difusco,sanokowski2024a,zhao2024disco}. While GNNs are effective for sparse, fixed topologies, their reliance on local message-passing can make it difficult to capture global dependencies and long-range interactions~\cite{Oono2020Graph,alon2021on}. Many problems involve general non-binary variables and dense global constraints whose natural representations are not well captured by a pairwise constraint graph.
To build a general-purpose neural heuristic,
the architecture must be able to model complex, global dependencies across diverse problem types.




\begin{figure}[t]
    \centering
    \includegraphics[width=\linewidth]{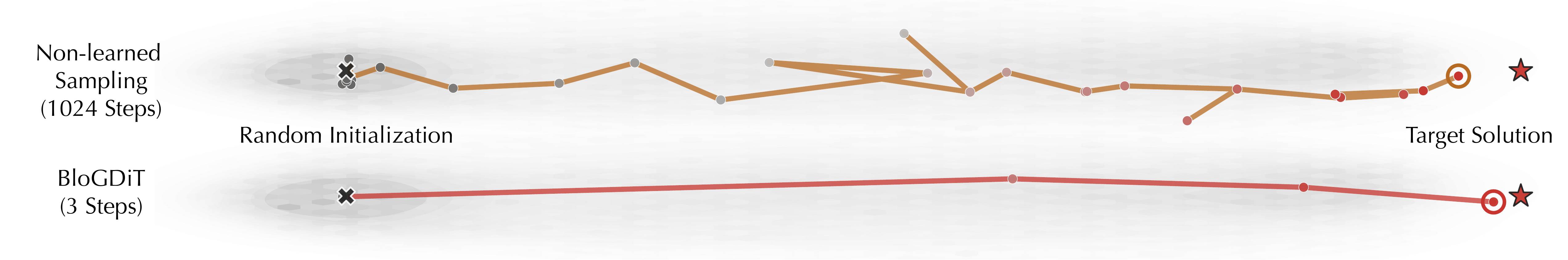}
    \caption{\textbf{Intuition for learned diffusion-guided refinement.}
We visualize assignment trajectories for a Sudoku instance in a 2D projection of the solution space. While the non-learned baseline~\cite{feng2025regularized} takes many exploratory updates (many of which are clustered around the initial assignment), \methodname{} reaches the target in far fewer denoising steps. This illustrates how a learned diffusion model amortizes the cost of search through directed iterative refinement. Projection details are provided in Appendix~\ref{app:traj_plot}.} 
    \label{fig:traj_intuition}
\end{figure}

\begin{wrapfigure}{r}{0.4\linewidth}
    \centering
    \includegraphics[width=\linewidth]{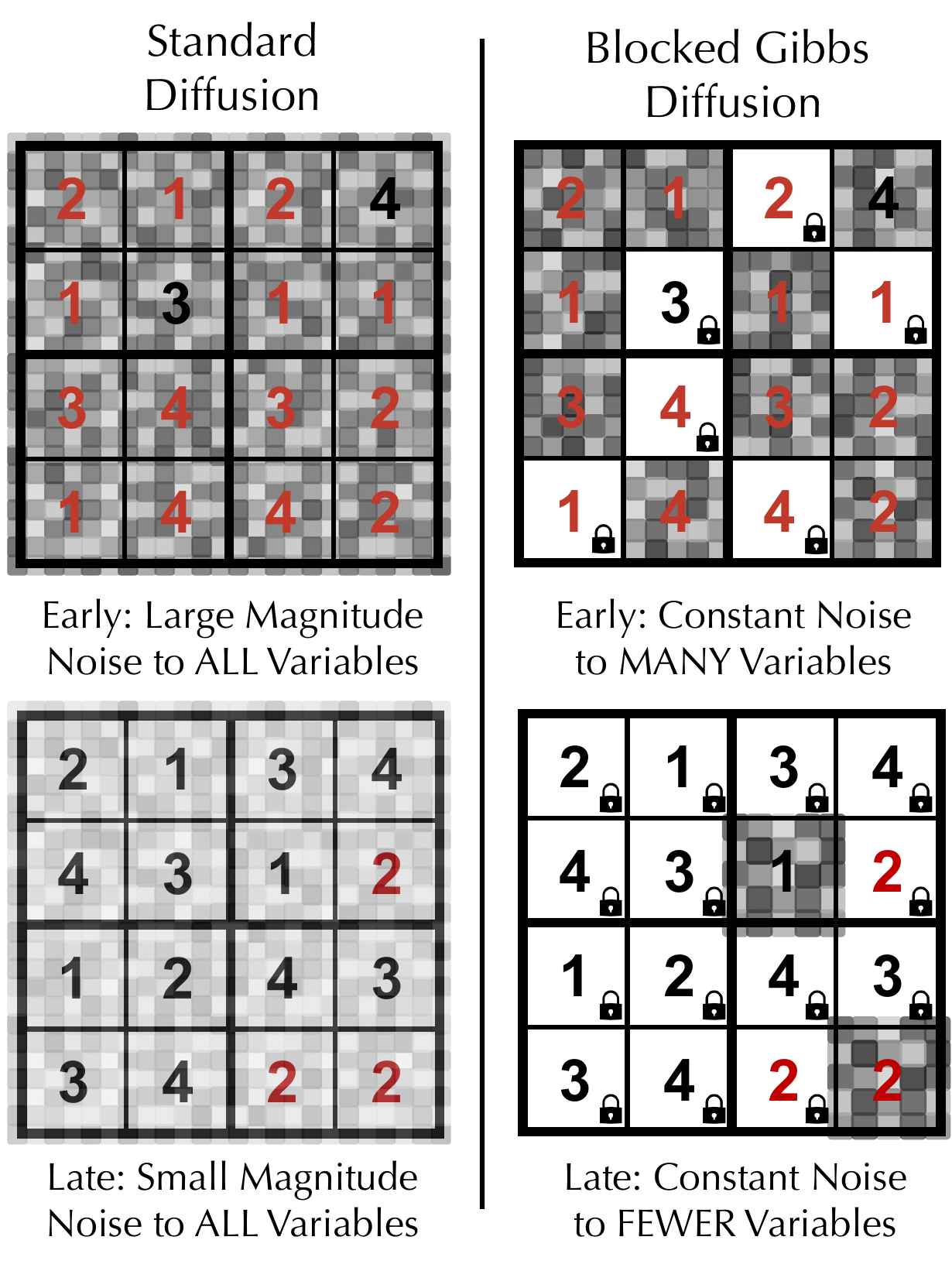}
    \caption{\textbf{Intuition for block-wise diffusion using a small Sudoku problem.} As a partial assignment becomes nearly feasible, maximum progress is made by resampling a small subset of variables, rather than applying small denoising updates everywhere.}
    \label{fig:sudoku_intuition}
\end{wrapfigure}
\textit{Diffusion Transformers} (DiTs) offer this capability~\cite{peebles2023dit}. With their global attention mechanism~\cite{vaswani2017attention}, Transformers can inherently model interactions between any pair of variables and have demonstrated superior scaling behavior in generative modeling~\cite{kaplan2020scaling}. 
As a starting point, we found that swapping out a backbone GNN for a DiT is insufficient. A naive application of DiT using standard ``global'' Gaussian diffusion schedules performs poorly on constraint problems. In preliminary experiments, a vanilla DiT generalized poorly and generated suboptimal solutions. We attribute this to a mismatch between global denoising and discrete optimization. In constraint satisfaction, improving a near-feasible solution typically requires \textit{substantially} changing \textit{only a small subset of variables } (e.g., resolving a specific constraint that is currently infeasible). Standard global diffusion, which updates every variable at every step, struggles to perform these surgical edits without disrupting the rest of the assignment.


To bridge this gap, we draw inspiration from \textit{blocked Gibbs sampling}~\cite{geman1984stochastic,gelfand1990sampling,jensen1995blocking} and recent work on Gibbs-style Diffusion~\cite{anil2025interleaved}. Rather than treating the entire solution as a single joint distribution to be denoised at once, we view the generation process as a sequence of conditional updates. This perspective aligns naturally with \textit{large neighborhood search}~\cite{lnsbook,shaw1998using}, a foundational heuristic in operations research in which solutions are improved by iteratively ``destroying'' (unassigning) a subset of variables and ``repairing'' (re-optimizing) them while keeping the remaining variables fixed.

\paragraph{Contributions.} Building on these insights, we propose \textbf{\methodname{}}, a \textbf{Blo}cked \textbf{G}ibbs \textbf{Di}ffusion \textbf{T}ransformer. Unlike standard diffusion models that operate in the full joint space, \methodname{} operates via iterative \emph{block resampling}. At each step, we apply continuous Gaussian noise to a specific \emph{subset} (block) of variables—effectively treating them as unknown or corrupted—and use the Transformer to denoise this block conditional on the remaining ``clean'' variables. This allows the model to focus its capacity on repairing sub-regions of the variable assignment. Crucially, we anneal the block size over time: early steps modify large blocks for global exploration, while later steps focus on smaller blocks for precise local refinement (\Cref{fig:sudoku_intuition}).  

To train~\methodname{}, we derive a mask-augmented joint KL upper bound for unsupervised training on Boltzmann targets and instantiate it with a masked Gaussian logit-space diffusion, yielding a tractable objective with closed-form entropy and forward noise-matching terms.

We evaluate \methodname{} on benchmark instances from four well-studied constraint satisfaction and optimization problems: Sudoku, Graph Coloring, Maximum Independent Set, and MaxCut. These problems vary in the types of constraints they involve, instance sizes, and variable domains (binary vs. categorical). Compared to a wide variety of methods that span the exact Google OR-Tools solver, pure sampling, and learning-based heuristics, \methodname{} achieves state-of-the-art or competitive performance, validating that the combination of Transformer backbones and blocked Gibbs updates provides an effective inductive bias for general constraint reasoning. 


\section{Background}
\subsection{Constraint Programming}
A \emph{constraint satisfaction problem} (CSP) is specified by a tuple
$(\mathcal{V}, D, C)$, where $\mathcal{V}=\{V^{(1)},\dots,V^{(n)}\}$ is a set of variables, $D=\{D^{(1)},\dots,D^{(n)}\}$ is a collection of finite discrete domains, and $C=\{c_1,\dots,c_m\}$ is a set of constraints. Each constraint $c_j$ is a relation over a subset of variables, 
restricting which joint assignments are allowed. Let $\mathcal{X} \triangleq \prod_{i=1}^n D^{(i)}$ denote the Cartesian product over variable domains, i.e., $X=(x^{(1)},\dots,x^{(n)}) \in \mathcal{X}$. The goal in a CSP is to find an assignment $X \in \mathcal{X}$ that satisfies all constraints:
$
c_j(X) = \texttt{true}, \forall j \in [m],
$
and a \emph{constraint optimization problem} (COP) adds an additional objective: $\min_{X \in \mathcal{X}} f(X)$.
\textit{Constraint programming} (CP) studies modeling languages and algorithms for CSPs/COPs,
with a focus on expressive \emph{global constraints} that compactly capture common
combinatorial structure~\cite{cp-handbook,globalconstraints,cappart2025combining}.
A common example is \textsc{AllDifferent}, which enforces that a subset of variables take on distinct values~\cite{alldifferent}.

\subsection{CP as sampling a Boltzmann distribution}
A probabilistic view of CP consists in defining an
\emph{energy} $H:\mathcal{X}\to\mathbb{R}$ whose low-energy assignments are desirable
(i.e., constraint-feasible), then considering the associated Boltzmann distribution
at temperature $\tau>0$:
\[
p_B(X) = \frac{\exp\left(-H(X)/\tau\right)}{Z(\tau)}, \text{where }
Z(\tau) = \sum_{X'\in\mathcal{X}} \exp\left(-H(X')/\tau\right).
\]
In the most general COP case, the energy combines the objective with constraint penalties:
\begin{equation}
H(X) = f(X) + \sum_{j=1}^m \lambda_j \phi_j(X),
\label{eq:cp_energy}
\end{equation}
where each $\phi_j(X)\ge 0$ satisfies $\phi_j(X)=0$ if and only if constraint
$c_j$ is satisfied, and $\lambda_j\ge 0$ is set to ensure feasible solutions are prioritized (e.g., a
Lagrangian-style relaxation~\cite{wolsey2020integer}). For CSPs, one may take $H(X)=\sum_j
\lambda_j \phi_j(X)$. As $\tau \to 0$, the Boltzmann distribution concentrates on
(assignments near) global minimizers of $H$. In practice, if one could efficiently sample from $p_B(X)$, then near-optimal solutions are likely to be found.

\subsection{Unsupervised diffusion on Boltzmann distributions}

In practice, sampling from $p_B(X)$ is often intractable due to the normalization factor $Z(\tau)$. A common approach is therefore to approximate $p_B(X)$ with a learned surrogate $q_\theta(X)$ that can be sampled from efficiently. Diffusion models provide one such approach. 
However, when data samples are not easily available, which is often the case for complex CSP/COP problems, the standard diffusion objective is not available. 
We take inspiration from recent approaches that apply diffusion-based learning to sample from Boltzmann densities in an unsupervised setting~\cite{sanokowski2024a,wang2025energy}.

Formally, let $q_\theta(X_0)$ denote the (generally intractable) model-implied marginal at time $0$, induced by a reverse-time Markov chain with parameters $\theta$.
We seek to minimize the reverse KL divergence, $\KL\left(q_\theta(X_0)\|p_B(X_0)\right)$.
As in~\citet{sanokowski2024a} and~\citet{wang2025energy}, we upper bound this marginal KL by a joint KL over trajectories
$X_{0:T}$:
\(
\KL\left(q_\theta(X_0)\|p_B(X_0)\right)
\le
\KL\left(q_\theta(X_{0:T})\|p(X_{0:T})\right)
\), 
where the forward (corruption) and reverse (generative) trajectory distributions are defined as
\begin{equation}
p(X_{0:T})
= p_B(X_0)\prod_{t=1}^{T} p(X_t \mid X_{t-1}), \quad
q_\theta(X_{0:T})
= q(X_T)\prod_{t=1}^{T} q_{\theta}(X_{t-1} \mid X_t), \label{eq:joint_q_diffuco}
\end{equation}
with $q(X_T)$ a fixed prior noise distribution 
and $\{p_t\}_{t=1}^T$ a fixed
forward noising process.

\section{Blocked Gibbs Diffusion}

\begin{figure*}[t]
    \centering
    \includegraphics[width=\linewidth]{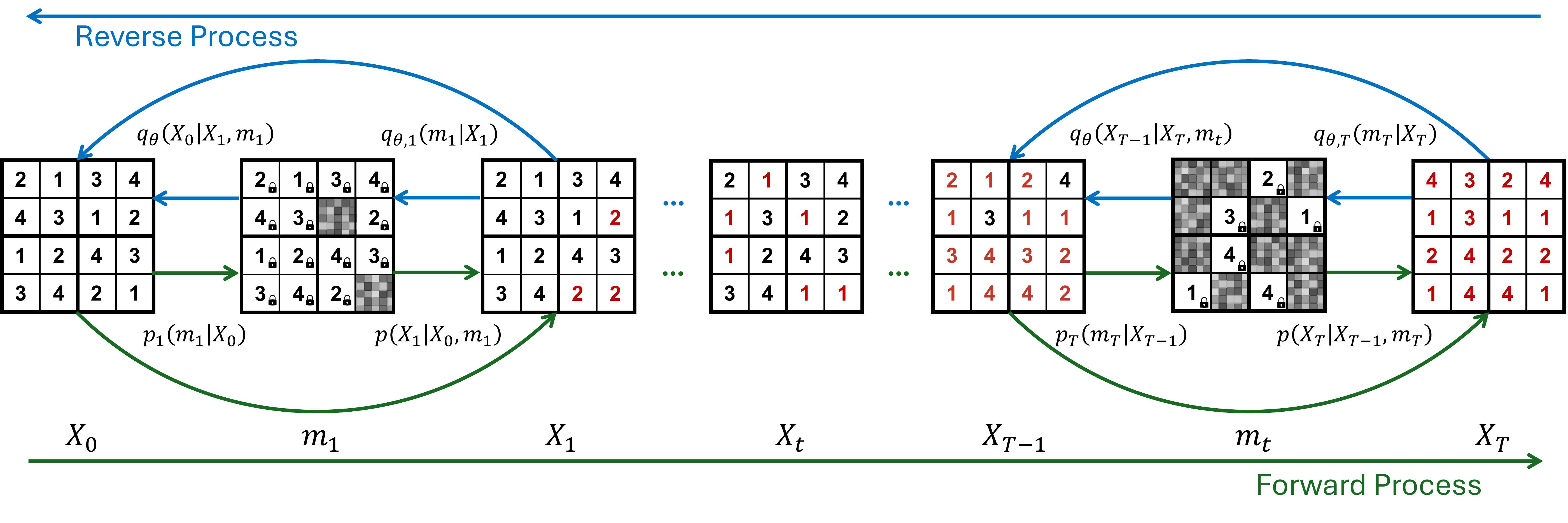}
    \caption{\textbf{Illustration of the blocked Gibbs diffusion process on a small Sudoku problem.}
The reverse chain (blue, top) is initialized from a high-noise Gaussian prior $Z_T \sim \mathcal{N}(0,s^2 I)$, which decodes to a random discrete assignment $X_T$. At step $t$, a binary mask $m_t$ is sampled to select a block of variables to update; the model samples $X_{t-1}^{m_t} \sim q_\theta(\cdot \mid X_t, m_t)$ while copying $X_{t-1}^{\neg m_t} = X_t^{\neg m_t}$ as-is, yielding blocked Gibbs-style resampling moves. The masking rate (block size) is annealed over time, transitioning from coarse global edits to fine local refinements. After $T$ steps, the final sample $X_0$ is expected to approximate the Boltzmann target
$p_B(X_0) \propto \exp(-H(X_0)/\tau)$.}

    \label{fig:placeholder}
\end{figure*}

To overcome limitations of existing diffusion models (cf. Section~\ref{sec:intro}), we advocate for a blocked Gibbs diffusion procedure 
that proposes
large, targeted resampling updates within the block. By annealing the block size over time, the diffusion transitions from coarse global edits to local refinements.
In this section, we define the mask-augmented forward and reverse processes underlying blocked Gibbs diffusion and derive its training objective. We describe our practical instantiation in \cref{sec:method_details}.

\subsection{Auxiliary mask variable}
\label{sec:aux_mask}
To model \emph{block} updates, we introduce an auxiliary binary mask $m_t$ at each step $t$.
Concretely, the $i$-th variable is associated with a bit mask $m_t^{(i)}\in\{0,1\}$: if $m_t^{(i)}=1$, $x^{(i)}$ is \emph{updated} at step $t$ (in both the forward and reverse transitions); if $m_t^{(i)}=0$, $x^{(i)}$ is \emph{frozen} and copied through the transition unchanged.

\paragraph{Mask-augmented transitions.}
We augment both the reverse and forward transition kernels with auxiliary $m_t$ variables:
\begin{align}
\label{eq:q_mask_mix}
q_{\theta}(X_{t-1}\mid X_t)
= \sum\nolimits_{m_t} q_{\theta,t}(m_t\mid X_t)
q_{\theta}(X_{t-1}\mid X_t,m_t), \\
p(X_t\mid X_{t-1})
= \sum\nolimits_{m_t} p_t(m_t\mid X_{t-1})
p(X_t\mid X_{t-1},m_t).
\end{align}
where $q_{\theta,t}(m_t\mid X_t)$ and $p_t(m_t\mid X_{t-1})$ are the reverse/forward mask-selection distributions.

The transition distributions leave unmasked variables unchanged. We define this property component-wise: for the $i$-th variable, the reverse transition is given by:
\begin{equation*}
q_{\theta}(x_{t-1}^{(i)} \mid X_t, m_t) =
\begin{cases}
\mathbb{I}\{x_{t-1}^{(i)} = x_t^{(i)}\} & \text{if } m_t^{(i)} = 0, \\
q_{\theta}\big(x_{t-1}^{(i)} \mid X_t\big) & \text{if } m_t^{(i)} = 1,
\end{cases}
\end{equation*}
where $\mathbb{I}\{\cdot\}$ is the indicator function.
The forward process $p(x_t^{(i)}\mid X_{t-1},m_t)$ is defined analogously, ensuring $x_t^{(i)} = x_{t-1}^{(i)}$ whenever the variable is unmasked.

\paragraph{Augmented joint distributions.}
We can equivalently define forward and reverse processes on the augmented space $(X_{0:T}, m_{1:T})$ by factorizing each step through the mask:
\begin{align}
p(X_{0:T}, m_{1:T}) 
& = p_B(X_0)\prod\nolimits_{t=1}^{T}
p_t(m_t\mid X_{t-1})\,p(X_t\mid X_{t-1},m_t),
\label{eq:joint_p}
\\
q_{\theta}(X_{0:T}, m_{1:T})
& = q(X_T)\prod\nolimits_{t=1}^{T}
q_{\theta,t}(m_t\mid X_t)\,q_{\theta}(X_{t-1}\mid X_t,m_t).
\label{eq:joint_q}
\end{align}
Importantly, introducing the auxiliary variables $m_{1:T}$ does not change the marginal trajectory distribution over $X_{0:T}$, as stated in Proposition~\ref{prop:1} with proof in Appendix~\ref{appendix:aux_proof}.

\begin{proposition}[Sampling equivalence under mask augmentation]
Let $(X_{0:T},m_{1:T})$ be sampled from the augmented joint distribution~\eqref{eq:joint_q}. By discarding $m_{1:T}$, the resulting trajectory
$X_{0:T}$ is distributed as the original process $q_\theta(X_{0:T})$ induced by the
mask-augmented transitions~\eqref{eq:q_mask_mix}:
\(
X_{0:T} \sim q_\theta(X_{0:T})
\Longleftrightarrow
\sum_{m_{1:T}} q_\theta(X_{0:T},m_{1:T}) = q_\theta(X_{0:T}).
\)
\label{prop:1}
\end{proposition}

\subsection{Mask-augmented joint upper bound}
Our new joint upper bound can be defined similarly as
\(
\KL\big(q_\theta(X_0)\|p_B(X_0)\big)
\le
\KL\big(q_\theta(X_{0:T},m_{1:T})\|p(X_{0:T},m_{1:T})\big).
\)
Substituting~\eqref{eq:joint_p}--\eqref{eq:joint_q} and collecting terms yields the following decomposition, as derived in full in Appendix~\ref{appendix:loss_derivation}:
\begin{equation}
\resizebox{\linewidth}{!}{$
\begin{aligned}
&\tau\KL\big(q_\theta(X_{0:T},m_{1:T})\|p(X_{0:T},m_{1:T})\big) 
=
\underbrace{\E_{q_\theta}\left[H(X_0)\right]}_{\text{Energy}}
-\tau\underbrace{
\sum\nolimits_{t=1}^{T}
\E_{X_{T:t},m_{T:t}\sim q_\theta}
\left[
h\big(q_{\theta}(X_{t-1}\mid X_t,m_t)\big)
\right]
}_{\text{Entropy}} \\
&\quad
-\tau\underbrace{
\sum\nolimits_{t=1}^{T}
\E_{X_{T:t},m_{T:t}\sim q_\theta}
\left[
\E_{X_{t-1}}
\log p(X_t\mid X_{t-1},m_t)
\right]
}_{\text{Forward noise matching}} 
+\tau\underbrace{
\E_{X_{0:T},m_{1:T}\sim q_\theta}
\left[
\sum\nolimits_{t=1}^{T}
\log
\frac{q_{\theta,t}(m_t\mid X_t)}
     {p_t(m_t\mid X_{t-1})}
\right]
}_{\text{Mask distribution matching}}
+ C .
\end{aligned}$}
\label{eq:main_loss_func}
\end{equation}
The first term on the right hand side is the energy term which encourages the final sample from the diffusion process $X_0$ to have low energy as defined in \Cref{eq:cp_energy}. The second term contains the differential entropy $h(q_{\theta}(X_{t-1}\mid X_t,m_t))$ over $q_{\theta}(X_{t-1}\mid X_t,m_t)$ which encourages diverse samples at each step of diffusion from the model. The third term ensures that the new samples at each step of the diffusion procedure match the known forward noise distribution; when $p_t(\cdot\mid\cdot,m_t)$ is chosen from a tractable family, the inner expectation is available in closed form. The final term ensures the mask distributions from the forward and reverse processes match.

\section{Blocked Gibbs Diffusion Transformer for Constraint Optimization}
\label{sec:method_details}
This section describes the practical instantiation of our masked reverse-time chain.
We use a continuous relaxation of the discrete assignments and perform diffusion \emph{in logit space}.
Assuming all variables share a common domain size $K \triangleq |D^{(i)}|$, we represent the diffusion state at time $t$ as a matrix of logits $Z_t \in \mathbb{R}^{n\times K}$, where row $i$ corresponds to variable $x^{(i)}$ and column $k$ corresponds to the $k$-th discrete value.
The relaxed assignment is obtained by applying a row-wise softmax,
$
\tilde{X}_t = \mathrm{softmax}(Z_t)\in[0,1]^{n\times K},
$
so that each row $\tilde{x}_{t}^{(i,:)}$ lies on the simplex $\Delta^{K-1}$.
We apply the forward corruption and reverse denoising transitions to $Z_t$, while energies and constraints are evaluated on the corresponding relaxed variables $\tilde{X}_t$.
At sampling time, we output a discrete assignment via $\hat{x}^{(i)} = \arg\max_k \tilde{x}_0^{(i,k)}$. For notational convenience, we identify matrices in $\mathbb{R}^{n\times K}$ with their vectorizations in $\mathbb{R}^{nK}$ and omit explicit $\operatorname{vec}(\cdot)$ operators.

\paragraph{Transformer architecture.}
We instantiate the denoiser with the Transformer backbone of~\citet{xu2025selfsupervised}, but perform diffusion over continuous logits rather than variable assignments, similar to~\citet{chen2023analog}. Instead of outputting the next assignment, we output the mean and variance of the logits,  which we then decode into variable assignments. The full Transformer architecture including its positional encodings and self-attention mechanism are described in Appendix~\ref{app:transformer}.

\paragraph{Energy function.}
We design constraint-specific energies by defining continuous, differentiable penalties $\phi_j(\tilde{X})\ge 0$ such that
$
\phi_j(\tilde{X})=0 \iff c_j(\hat{X})=\texttt{true},
$
where $\hat{X}$ denotes the discretization induced by $\tilde{X}$ (e.g., row-wise argmax).
The overall energy is the weighted sum of the objective and penalties as in \Cref{eq:cp_energy}.

\paragraph{Noise distribution.}


We employ a masked Gaussian forward step kernel:
\[
p(Z_t \mid Z_{t-1}, m_t)
=
\mathcal{N}\Big(Z_{t-1},\sigma^2\mathrm{diag}(\tilde m_t)\Big),
\]
where $\sigma>0$ is fixed, and $\tilde m_t\in\{0,1\}^{n\times K}$ is the variable-level mask broadcast to the logit dimensions, i.e.,
$\tilde m_t^{(i,k)}=m_t^{(i)}$.
Equivalently,
$
Z_t = Z_{t-1} + (\tilde m_t \odot \epsilon_t)
$ where $\epsilon_t \sim \mathcal{N}(0,\sigma^2 I),
$
with unmasked coordinates being copied as-is. Repeated application yields a marginal variance that grows with $t$ on the masked coordinates.

We parameterize the masked reverse step with a diagonal-covariance Gaussian:
\[
q_{\theta}(Z_{t-1}\mid Z_t, m_t) = \mathcal{N}\big(\mu_{\theta}(Z_t,m_t),\mathrm{diag}(\sigma_{\theta}^2(Z_t,m_t))\big),
\]
where the model outputs $\mu_{\theta}\in\mathbb{R}^{n\times K}$ and $\log\sigma_{\theta}^2\in\mathbb{R}^{n\times K}$.
To enforce exact copying on unmasked coordinates, we set
$
\mu_{\theta}^{\neg \tilde m_t} = Z_t^{\neg \tilde m_t}, \sigma_{\theta}^{\neg \tilde m_t}=0,
$
and only sample the masked coordinates. The corresponding $X_{t-1}$ is obtained by applying per-variable softmax to $Z_{t-1}$.

\paragraph{Closed form for entropy and noise-matching term.}
The differential entropy of a diagonal Gaussian restricted to the masked coordinates admits the closed form
\[
h\left[q_{\theta}(Z_{t-1}\mid Z_t,m_t)\right]
=
\frac{1}{2}
\sum_{i=1}^{n}
\sum_{k=1}^{K}
m_t^{(i)}
\Big(
\log(2\pi e)
+
\log \sigma_{\theta,i,k}^2(Z_t,m_t)
\Big).
\]
where only masked dimensions contribute.
Under the masked Gaussian forward kernel,
the inner expectation in the forward noise-matching term has a closed form:
\[
\E_{q_{\theta}(Z_{t-1}\mid Z_t,m_t)}
\big[-\log p_t(Z_t\mid Z_{t-1},m_t)\big]
=
\frac{1}{2\sigma^2}\Big(
\|\tilde m_t \odot (Z_t-\mu_\theta)\|_2^2
+
\mathrm{tr}(\mathrm{diag}(\tilde m_t)\Sigma_\theta)
\Big)
+\mathrm{const}.
\]
where $\Sigma_\theta=\mathrm{diag}(\sigma_\theta^2(Z_t,m_t))$ is the covariance matrix and $\mathrm{const}$ does not depend on $\theta$.

\paragraph{Base distribution at $Z_T$.}
We initialize the chain from a fixed Gaussian base distribution:
$
Z_T \sim q(Z_T), q(Z_T)=\mathcal{N}(0, s^2 I),
$
where $s$ is a fixed hyperparameter (optionally with a constant bias on selected coordinates).
In constrained settings with given variable values (e.g., Sudoku, where some variables are fixed), we clamp the corresponding coordinates and sample the remaining ones,
which can be viewed as drawing from a conditional base distribution $q(Z_T \mid \text{givens})$.

\paragraph{Mask distribution.}
At each step $t \in [T]$, we sample a mask $m_t \in \{0,1\}^n$ indicating the subset of variables to update.
We choose a \emph{time-dependent, state-agnostic} mask schedule shared by forward and reverse processes, 

\begin{equation}
q_{\theta,t}(m_t\mid X_t) \equiv p_t(m_t\mid X_{t-1})\equiv \pi_t(m_t)
=\prod_{i=1}^n \mathrm{Bernoulli}\big(m_t^{(i)};\rho_t\big).
\label{eq:mask_bern}
\end{equation}
with masking rate $\rho_t$. We parameterize the block schedule by a minimum and maximum masking rate $[\rho_{\min},\rho_{\max}]$ and anneal $\rho_t$ over time to transition from large to small block updates; exact schedules are given in Appendix~\ref{app:mask_schedule}.
Because we use the same schedule for both forward and reverse mask variables, the last term in \Cref{eq:main_loss_func} vanishes.



\paragraph{Adaptive mask selection.}
\label{sec:method_adaptive}
A practical advantage of the blocked Gibbs formulation is that it exposes an explicit mask-selection step: at each reverse transition during inference, the sampler can decide which subset of variables to update. This makes it possible to replace random block selection with adaptive, state-dependent selection strategies that focus on variables most likely to benefit from updating. In contrast, a standard DiT updates all variables jointly at every denoising step and therefore has no analogous mechanism for prioritizing variables.
We use three representative strategies. First is a confidence-margin strategy, inspired by adaptive token selection in masked diffusion~\cite{kim2025train}, which prioritizes variables for which the model is uncertain. The second and third are inspired by classical Large Neighborhood Search~\cite{lnsbook,shaw1998using}: a violation-based strategy, which prioritizes variables involved in violated constraints, and a related-variable strategy, which selects variables that participate in a common constraint.
 We treat the choice of selection strategy as a hyperparameter and report the best-performing strategy for each benchmark; formal definitions are provided in Appendix~\ref{app:adaptive_sampling}.

\section{Experiments}
The experiments seek to answer the following research questions:
\textbf{RQ1:} Is diffusion a useful learning paradigm compared to simpler iterative Transformer baselines?
\textbf{RQ2:} Does blocked Gibbs diffusion improve over standard diffusion Transformers?
\textbf{RQ3:} Does using the more general Transformer architecture degrade performance relative to methods that are specialized for binary problems?
\textbf{RQ4:} Can \methodname{} outperform non-learned state-of-the-art methods? Our code is available at \url{https://github.com/khalil-research/BloGDiT}.

\subsection{Problem selection}

We evaluate \methodname{} on a diverse set of benchmarks from the literature, following their respective evaluation conventions. For each problem, we define a CP formulation (Appendix~\ref{app:cpform}) and a continuous relaxation penalty for each constraint (\Cref{tab:loss_relaxed}), resulting in a problem-specific energy for training. For COPs, we use a weighted sum of the constraint penalties and the objective (cf. \cref{eq:cp_energy}).

\begin{wraptable}{r}{0.5\linewidth}
    \centering
    \caption{
    Discrete constraints $c(X)$ used in our four problems and their continuous penalty functions, $\phi_c(X)$. In the continuous penalties, each discrete variable assignment $x^{(i)}$ is represented by a probability vector approximating its one-hot encoding, with entries $\tilde{x}^{(i,k)}\in[0,1]$ for $k\in[K]$, where $K=|D^{(i)}|$ is the common domain size.}
    \resizebox{\linewidth}{!}{%
    \begin{tabular}{ll}
        \toprule
        \textbf{Discrete Constraint $c(X)$} & \textbf{Continuous Penalty $\phi_c(\tilde{X})$} \\ \midrule

        $\textsc{AllDifferent}_{m=n}(x^{(1)}, \ldots, x^{(n)})$ & $\sum_{k=1}^{K} \left( \left |1 - \sum_{i=1}^n \tilde{x}^{(i,k)} \right | \right)$ \\ \midrule

        $x^{(i)} \neq x^{(j)}$ & $\sum_{k=1}^{K}(\tilde{x}^{(i,k)} \cdot \tilde{x}^{(j,k)})$ \\ 
        \midrule

        $\forall (i,j)\in E:\; x^{(i)} + x^{(j)} \le 1$
        &
        $\sum_{(i,j)\in E}\left(\tilde{x}^{(i,1)} \cdot \tilde{x}^{(j,1)}\right)$ \\
        \bottomrule
    \end{tabular}
    }
    \label{tab:loss_relaxed}
\end{wraptable}

\textbf{Sudoku} is a CSP defined on a $9\times 9$ grid with variables taking values in $\{1,\dots,9\}$.
The goal is to fill the empty cells such that each row, column, and $3\times 3$ sub-grid contains each digit exactly once. A single instance consists of a partially filled board (``givens'') and a completion that satisfies all constraints.
Instance difficulty is determined by the number and placement of givens. We encode Sudoku using $\textsc{AllDifferent}$ constraints over rows, columns, and sub-grids. Following \citet{xu2025selfsupervised}, we train and evaluate on the SATNet dataset~\cite{satnet} of instances with 31--42 given cells, and test for \textit{out-of-distribution} (OOD) generalization on the RRN dataset~\cite{rrn} of harder instances with 17--34 given cells.

\textbf{Graph Coloring} asks for an assignment of $k$ colors to vertices such that adjacent vertices have different colors. We generate two sets of instances for $k=5$ and $k=10$. Training graphs have 50 vertices for $k=5$ and 100 vertices for $k=10$. To test OOD generalization, we evaluate on larger graphs with 100 vertices for $k=5$ and 200 vertices for $k=10$. We use pairwise not-equal constraints $x^{(i)} \neq x^{(j)}$ for each edge $(i,j)$ to capture the coloring requirement.

\textbf{MaxCut} is a COP that partitions the vertices of a graph into two sets such that the total count of cut edges is maximized. We train on Erdős--Rényi graphs with \(n=50\) and edge probability  sampled uniformly from \((0.05,0.3)\). We evaluate large-scale generalization on benchmark instances from the GSET dataset~\cite{ye2003gset}, which consists of graphs with 800--10,000 vertices. Following existing works, we report the gap to the best known cut sizes~\cite{matsuda2019benchmarking}.

\textbf{Maximum Independent Set (MIS)} is a COP that seeks the largest subset of vertices with no pairwise edges between them. Following~\citet{sanokowski2024a}, we generate graphs using the RB-Model~\cite{rbmodel}. We report results on RB-small and RB-large with graph sizes between 200-300 and 800-1200 nodes.


\subsection{Results}

\Cref{tab:combined_results} summarizes our main results across problems. 
We compare against classical solvers, learning-based and non-learning-based heuristics. Baseline descriptions and sources of reported results are provided in Appendix~\ref{app:baselines}, while computational protocols and their differences are discussed in Appendix~\ref{app:efficiency}. 
The best and second-best values are bolded and underlined, respectively. DiT denotes a vanilla diffusion Transformer with standard Gaussian denoising over all variables. \methodname{} denotes our method, using blocked Gibbs diffusion with adaptive mask selection as described in \Cref{sec:method_adaptive}. OR-Tools is a state-of-the-art exact CP solver and serves as a strong baseline~\cite{cpsat}, achieving 100\% on Sudoku instances. OR-Tools uses tree search along with sophisticated heuristics and thus should be seen as representative of the best algorithmic practices from the CP community; indeed, it has won gold in the international constraint programming competition every year since 2013~\cite{stuckey2014minizinc}.

\begin{table}[t]
\centering
\caption{
Performance comparison.
Higher is better for MIS (average objective value), Sudoku, and Graph Coloring (\% of instances solved). 
Lower is better for MaxCut (average absolute gap to best-known solution).
Bold indicates the best result and underline indicates the second-best result.
}
\label{tab:combined_results}
\resizebox{0.97\textwidth}{!}{
\begin{tabular}{@{}c@{\hspace{1.2em}}c@{}}
\\ \toprule
\resizebox{0.37\textwidth}{!}{
\begin{tabular}[t]{@{}lcc@{}}
\textbf{Sudoku} & Easy $\uparrow$ & Hard $\uparrow$ \\
\midrule
OR-Tools & \textbf{100} & \textbf{100} \\
\midrule
SATNet & 98.3 & 3.2 \\
RRN & 99.8 & 28.6 \\
\citet{yang2023learning} & \textbf{100} & 32.9 \\
IRED & 99.4 & 62.1 \\
RLSA & 4.2 & 1.8 \\
ConsFormer & \textbf{100} & 77.7 \\
\midrule
DiT & \textbf{100} & 48.6 \\
\methodname{} & \textbf{100} & \underline{94.1} \\
\end{tabular}
}
&
\resizebox{0.7\textwidth}{!}{
\begin{tabular}[t]{@{}lcccc@{}}
\multirow{2}{*}{\textbf{Graph Coloring}}& \multicolumn{2}{c}{5 Colors}
& \multicolumn{2}{c}{10 Colors} \\
\cmidrule(lr){2-3} \cmidrule(lr){4-5}
& $n=50$ $\uparrow$ & $n=100$ $\uparrow$
& $n=100$ $\uparrow$ & $n=200$ $\uparrow$ \\
\midrule
OR-Tools & \textbf{83.08} & \textbf{57.16} & 52.41 & 10.25 \\
\midrule
Greedy & 32.42 & 0.00 & 0.75 & 0.00\\
Feasibility Jump & 82.83 & 54.50 & 35.66 & 6.00 \\ 
ANYCSP & 79.17 & 34.83 & 0.00 & 0.00 \\
RLSA & 37.83 & 0.42 & 10.75 & 2.50 \\
ConsFormer & 81.00 & 47.33 & \underline{52.60} & 11.92 \\
\midrule
DiT & 79.75 & 48.91 & 50.50 & \underline{14.25} \\
\methodname{} & \underline{83.00} & \underline{54.67} & \textbf{53.92} & \textbf{18.58} \\
\end{tabular}
}
\\ \midrule
\midrule
\resizebox{0.37\textwidth}{!}{
\begin{tabular}[t]{@{}lcc@{}}
\textbf{MIS} & RB-small $\uparrow$ & RB-large $\uparrow$ \\
\midrule
OR-Tools & \textbf{20.10} & \textbf{42.66} \\
\midrule
LwD & 19.01 & 32.32 \\
INTEL & 18.47 & 34.47 \\
DGL & 17.36 & 34.50 \\
LTFT & 19.18 & 37.48 \\
DIFUSCO & 18.52 & -- \\
RLSA & 19.97 & \underline{40.19} \\
\textcolor{black}{X$^2$GNN} & \underline{20.07} & 39.28 \\
DiffUCO & 18.88 & 38.10 \\
\midrule
DiT & 8.05 & 0.29 \\
\methodname{} & 19.52 & 37.05 \\
\end{tabular}
}
&
\resizebox{0.7\textwidth}{!}{
\begin{tabular}[t]{@{}lcccc@{}}
\textbf{MaxCut} & $|V|=800$ $\downarrow$ & $|V|=1K$ $\downarrow$
& $|V|=2K$ $\downarrow$ & $|V|\geq3K$ $\downarrow$ \\
\midrule
OR-Tools & 143.89 & 112.78 & 365.89 & 378.62 \\
\midrule
RUNCSP & 185.89 & 156.56 & 357.33 & 401.00 \\
ECO-DQN & 65.11 & 54.67 & 157.00 & 428.25 \\
ECORD & 8.67 & 8.78 & 39.22 & 187.75 \\
ANYCSP & \underline{1.22} & \underline{2.44} & 13.11 & 51.63 \\
RLSA & 3.22 & 4.00 & \underline{13.0} & 66.75 \\
\textcolor{black}{X$^2$GNN} & 21.11 & 6.00 & 22.44 & \underline{46.38}\\
DiffUCO & 4.11 & 6.33 & 31.67 & 116.75 \\
ConsFormer & 16.33 & 12.44 & 52.11 & 115.25 \\
\midrule
DiT & 56.11 & 50.00 & 101.33 & 162.00 \\
\methodname{} & \textbf{0.11} & \textbf{1.33} & \textbf{7.88} & \textbf{32.75} \\
\end{tabular}
}
\\
\bottomrule
\end{tabular}
}
\end{table}

\paragraph{RQ1: Diffusion vs. iterative Transformers.}
We compare with ConsFormer~\cite{xu2025selfsupervised} as the baseline (non-diffusion) iterative Transformer model which achieved state-of-the-art result on OOD Sudoku tasks. 
\methodname{} consistently improves over ConsFormer, with the largest gains in harder OOD settings: Sudoku OOD improves from 77.7\% to 94.1\%, and Graph Coloring improves across both $k=5$ and $k=10$, especially OOD. On MaxCut, \methodname{} also substantially improves over ConsFormer across all GSET size buckets. 
These results suggest that the diffusion objective provides a more principled framework for learning iterative heuristics, by grounding training in a reverse-KL objective toward a Boltzmann target rather than directly training a Transformer to improve assignments.

\textbf{RQ2: Blocked Gibbs vs. standard diffusion.}
We compare against the standard DiT baseline, which uses the same Transformer backbone, training data, unsupervised objective, optimization settings, and inference budget but applies Gaussian denoising to all variables at each step. This isolates the effect of replacing global denoising with block-wise resampling. While DiT is competitive in some in-distribution settings, it degrades more sharply on harder instances: Sudoku OOD drops to 48.6\% compared to 94.1\% for \methodname{}, and MIS reaches average independent set sizes of only $8.05$ and $0.29$ on RB-small and RB-large, compared to $19.52$ and $37.05$ for \methodname{}. \methodname{} also improves over DiT on OOD Graph Coloring and reduces MaxCut gaps across all GSET sizes. These trends support our hypothesis that constraint optimization benefits from targeted block-wise resampling, where most variables can be preserved while selected variables are substantially revised.

\textbf{RQ3: Transformer vs graph-specialized baselines.}
We compare against learned baselines for binary graph optimization, such as DIFUSCO~\cite{sun2023difusco}, DiffUCO~\cite{sanokowski2024a} and X$^2$GNN~\cite{acikalin2025learning}, which incorporate graph-specific inductive biases. These methods are designed for binary graph problems, so this comparison tests whether our more general Transformer-based framework sacrifices performance for broader applicability. On MIS, \methodname{} remains competitive but does not uniformly dominate the graph-specialized methods. In contrast, on MaxCut, \methodname{} achieves the best performance across all graph sizes.
These results suggest that \methodname{} can remain competitive on binary graph optimization while also supporting non-binary CSPs such as Sudoku and Graph Coloring.
\textcolor{black}{To further examine the effect of the transformer backbone, we additionally experiment with a GCN architecture (\Cref{tab:ablations_big}). The Transformer achieves substantially better results on larger instances, particularly MIS-RB-large (37.05 vs.\ 19.99) and MaxCut with $|V|\geq 3$K (32.75 vs.\ 79.75 gap), suggesting that the Transformer backbone can benefit constraint optimization.
}

\textbf{RQ4: Comparison with non-learned solvers.}
We compare \methodname{} against state-of-the-art non-learned solvers, including OR-Tools~\cite{cpsat} and RLSA~\cite{feng2025regularized}. On Sudoku, OR-Tools achieves $100\%$ accuracy on both splits and outperforms all learned methods. Nevertheless, \methodname{} matches OR-Tools on the easy split and achieves the strongest learned-method result on the harder OOD split, with $94.1\%$ accuracy. On Graph Coloring, \methodname{} is competitive with OR-Tools for $k=5$ and outperforms it for $k=10$, improving accuracy from $10.25\%$ to $18.58\%$ on the hardest instances. RLSA performs poorly on both Sudoku and Graph Coloring, which is expected since it is designed primarily for binary optimization rather than non-binary CSPs. On MIS, OR-Tools obtains the strongest results overall.
On MaxCut, however, \methodname{} outperforms both OR-Tools and RLSA across all GSET size buckets. Overall, no method is universally best. \methodname{} provides a learned 
heuristic that performs strongly across both non-binary CSPs and large-scale binary graph optimization tasks, remaining competitive with or improving over non-learned baselines.

\subsection{Ablations}
\label{sec:ablate}

\paragraph{Masking rate.}
We ablate the mask schedule by sweeping the masking rates $\rho_{\max}$ and $\rho_{\min}$. The results are shown in Table~\ref{tab:mask-schedule-ablation}. The main takeaway is that the schedule should start with large updates and gradually anneal to smaller updates. Small initial update ratios lead to poor performance, suggesting that large early updates are needed to make global changes. 
At the same time, keeping the update ratio too large until the end is also harmful for generalization: $(\rho_{\max},\rho_{\min})=(0.9,0.7)$ performs well on RB-Small but collapses on RB-Large, indicating that late-stage updates should become smaller to allow fine-grained repair. We extend this analysis in Appendix~\ref{app:realized_change}.

\begin{table*}[t]
\centering

\begin{minipage}[t]{0.58\textwidth}
\centering
\caption{Ablation results. \textit{$-$AMS} uses random mask
selection, \textit{$-$Entropy} trains without the entropy term,
and \textit{w/ GCN} replaces the Transformer backbone with a
parameter-matched GCN. Sudoku and Graph Coloring
values are percentages of solved instances ($\uparrow$).}
\label{tab:ablations_big}
\resizebox{\linewidth}{!}{
\begin{tabular}{lrrrr}
\toprule
Dataset & \methodname{} & \textit{$-$AMS} &\textit{ $-$Entropy} & \textit{w/ GCN} \\
\midrule
Sudoku-Easy
    & \textbf{100.00} & \textbf{100.00} & \textbf{100.00} & \textbf{100.00} \\
Sudoku-Hard
    & \textbf{94.12} & 88.21 & 91.92 & 91.00 \\
Coloring-5 ($n=50$)
    & \textbf{83.00} & 82.92 & \textbf{83.00} & 82.83 \\
Coloring-5 ($n=100$)
    & \textbf{54.67} & 54.58 & 53.42 & 53.00 \\
Coloring-10 ($n=100$)
    & \textbf{53.92} & \textbf{53.92} & \textbf{53.92} & \textbf{53.92} \\
Coloring-10 ($n=200$)
    & \textbf{18.58} & 17.42 & 18.25 & 18.00 \\
MIS-RB-small (size $\uparrow$)
    & \textbf{19.52} & 15.83 & 18.99 & 19.05 \\
MIS-RB-large (size $\uparrow$)
    & \textbf{37.05} & 17.16 & 35.03 & 19.99 \\
MaxCut-800 (gap $\downarrow$)
    & \textbf{0.11} & 2.11 & 1.22 & 4.67 \\
MaxCut-1K (gap $\downarrow$)
    & \textbf{1.33} & 3.22 & 2.67 & 9.56 \\
MaxCut-2K (gap $\downarrow$)
    & \textbf{7.88} & 18.89 & 13.33 & 28.33 \\
MaxCut-$\geq$3K (gap $\downarrow$)
    & \textbf{32.75} & 88.63 & 61.00 & 79.75 \\
\bottomrule
\end{tabular}}
\end{minipage}
\hfill
\begin{minipage}[t]{0.39\textwidth}
\centering
\caption{Masking rate ablation on MIS. The masking rate is annealed from $\rho_{\max}$ to $\rho_{\min}$ following a geometric schedule (cf. Appendix \ref{app:mask_schedule}). Values are mean independent set sizes.}
\label{tab:mask-schedule-ablation}
\resizebox{\linewidth}{!}{
\begin{tabular}{ccrr}
\toprule
$\rho_{\max}$ & $\rho_{\min}$ & RB-Small & RB-Large \\
\midrule
0.3 & 0.1 & 5.45 & 6.78 \\
0.5 & 0.1 & 5.21 & 9.93 \\
0.5 & 0.3 & 4.78 & 6.14 \\
0.7 & 0.1 & 12.46 & 33.17 \\
0.7 & 0.3 & 14.80 & 29.30 \\
0.7 & 0.5 & 17.01 & 31.05 \\
0.9 & 0.1 & 18.67 & 34.62 \\
0.9 & 0.3 & \textbf{19.52} & \textbf{37.05} \\
0.9 & 0.5 & \underline{19.49} & \underline{35.41} \\
0.9 & 0.7 & 18.58 & 4.24 \\
\bottomrule
\end{tabular}
}
\end{minipage}

\end{table*}

\textbf{Adaptive mask selection.}
\Cref{tab:ablations_big} ablates adaptive mask selection (AMS) by replacing it with random selection while keeping the same masking rate. AMS consistently improves performance, with the largest gains appearing on harder optimization settings. On Sudoku-Hard, AMS improves accuracy from $88.21\%$ to $94.12\%$. On Graph Coloring, the gains are smaller but consistent. The effect is much larger on MIS and MaxCut: on MIS-RB-large, AMS improves the average independent set size from $17.16$ to $37.05$, and consistently reduces the average optimality gap for all MaxCut sizes. These results indicate that blocked diffusion benefits substantially from selecting which variables to update, rather than choosing blocks blindly.

\textbf{Role of entropy.}
\label{sec:method_greedy}
Existing CSP/COP benchmarks often reward finding a single high-quality solution, rather than producing a diverse set of samples. This raises a natural question about the role of the entropy term in our training objective (\cref{eq:main_loss_func}), since this term encourages sample diversity. As \citet{sanokowski2024a} note, annealing the Boltzmann temperature toward zero shifts the objective toward energy minimization, reducing the relative importance of entropy in the low-temperature limit.
To test whether entropy remains useful in this optimization-focused setting, we train \methodname{} with the entropy term removed. We find that removing entropy generally degrades performance relative to the full objective (\cref{tab:ablations_big}). This suggests that, although entropy is not part of the energy objective used to evaluate solution quality, it plays an important regularizing role during diffusion training by maintaining exploration and preventing premature collapse to suboptimal regions.

\begin{wrapfigure}{r}{0.52\textwidth}
    \centering
    \includegraphics[width=\linewidth]{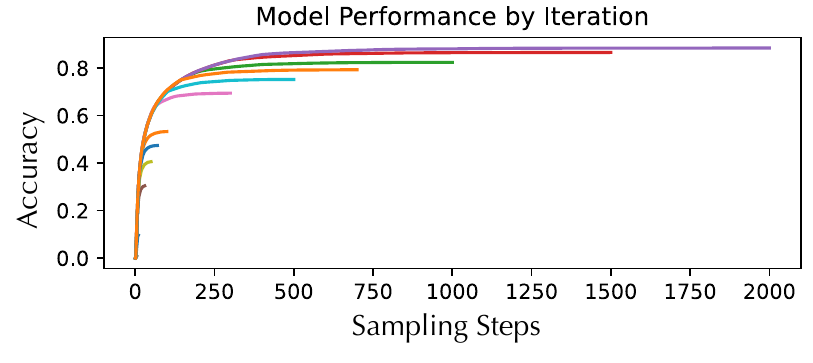}
    \caption{Sudoku test-time accuracy as a function of the number of sampling steps $T$. Each curve corresponds to a different number of sampling steps with the same model.}
    \label{fig:sampling_steps}
\end{wrapfigure}

\paragraph{Sampling steps.}
We ablate the number of sampling steps used at test time while keeping training and all other hyperparameters fixed.
We run \methodname{} for $T\in\{25,50,100,200,500,1000,2000\}$ steps and report the resulting solution accuracy on Sudoku. \Cref{fig:sampling_steps} shows that performance improves rapidly in early iterations and exhibits diminishing returns beyond a few hundred steps, i.e., most of the quality gains are achieved with relatively short sampling trajectories, suggesting a trade-off between computation cost and solution quality. We further examine this effect in Appendix~\ref{app:ablate-T}



\section{Related work}

\textbf{Learning-based approaches for discrete optimization.}
A large body of work studies learning heuristics for constraint problems, including GNN-based  methods~\cite{anycsp,runcsp} as well as Transformer-based models for constraint satisfaction~\cite{xu2025selfsupervised,yang2023learning}.
More recently, diffusion models have emerged as an alternative paradigm for combinatorial optimization, with prominent lines focusing on graph-structured problems~\cite{sun2023difusco} and unsupervised diffusion frameworks that optimize solutions via energy guidance without requiring supervised solution datasets~\cite{sanokowski2024a,sanokowski2025scalable,zhao2024disco}.
In contrast, our goal is to move beyond specifically graph-based backbones by using a Diffusion Transformer.

\textbf{Gibbs-style sampling in diffusion.}
A growing body of work investigates diffusion processes that iteratively edit subsets of variables. SDEdit~\cite{meng2022sdedit} demonstrates the effectiveness of re-noising and re-denoising for editing in continuous domains, while remasking methods develop principled inference-time strategies for discrete diffusion models~\cite{wang2025remasking}. Closely related to our approach, Interleaved Gibbs Diffusion~\cite{anil2025interleaved} implements generation as Gibbs-style conditional updates, and Block Diffusion~\cite{arriola2025block} studies blockwise interpolation between autoregressive and diffusion generation in language modeling.
We provide a more detailed discussion of related work in Appendix~\ref{app:extended-related-work}.

\section{Conclusion}

We introduced \methodname{}, a Blocked Gibbs Diffusion Transformer for unsupervised constraint satisfaction and optimization via Boltzmann sampling. Our key insight is that near-feasible solutions typically require sparse edits to a subset of variables, which is poorly matched to standard diffusion updates that denoise all variables at every step. \methodname{} addresses this mismatch by performing block-wise resampling moves with an annealed block-size schedule, enabling global exploration early and localized refinement late. Across four benchmarks, \methodname{} matches or improves over strong neural baselines, suggesting that Gibbs-style conditional updates provide a useful inductive bias for Transformer-based constraint reasoning.

\paragraph{Limitations and future work.}
\label{sec:future}
An important next step is to incorporate scalability techniques such as those of \citet{sanokowski2025scalable} to reduce training and sampling cost. We operated in the continuous logit space in this work; further investigation of discrete diffusion methods that act directly on the discrete domain, as in \citet{sanokowski2024a}, would be insightful. In addition, while our approach is motivated by minimizing the reverse KL to a Boltzmann target and produces high-quality approximate samples in practice, it does not provide exact sampling guarantees. Future work could follow recent trends in diffusion~\cite{du2023reduce,sjoberg2023mcmc} and add an MCMC-style correction (e.g., a Metropolis–Hastings accept–reject step~\cite{hastings1970monte}) to more faithfully target the true Boltzmann distribution.

\clearpage

\section*{Acknowledgments}
This work was supported by the Institute of Information \& communications Technology Planning \& Evaluation (IITP) grant funded by the Korean government(MSIT) (No. RS-2024-00457882, National AI Research Lab Project).



\bibliographystyle{plainnat}
\bibliography{references}


\newpage
\appendix



\section{Loss derivation}
\label{appendix:loss_derivation}

To derive the objective detailed in \Cref{eq:main_loss_func}, we first expand the reverse KL divergence over the joint distribution:
\begin{align}
&\KL\big(q_\theta(X_{0:T},m_{1:T})\|p(X_{0:T},m_{1:T})\big).\\
& \nonumber \quad = \E_{q_\theta}\left[
\log q(X_T) - \log p_B(X_0) \right] \\
& \nonumber \qquad + \sum_{t=1}^T \E_{q_\theta}\left[
\log q_\theta(X_{t-1}\mid X_t,m_t)
- \log p(X_t\mid X_{t-1},m_t)\right] \\
&\qquad + \sum_{t=1}^T \E_{q_\theta}\left[
\log q_{\theta,t}(m_t\mid X_t) - \log p_t(m_t \mid X_{t-1})\right]\\
& \nonumber \quad = - \; \underbrace{\log p_B(x_0)}_{\text{(1) Energy}} \\
&\qquad +\underbrace{\E_{X_{0:T},m_{1:T} \sim q_\theta}\left[\sum_{t=1}^{T}
\log q_{\theta}(X_{t-1}\mid X_t, m_t)\right]}_{\text{(2) Reverse entropy / reconstruction}} \\
&\qquad-\underbrace{\E_{X_{0:T},m_{1:T} \sim q_\theta}\left[\sum_{t=1}^{T}
\log p(X_t\mid X_{t-1}, m_t)\right]}_{\text{(3) Forward noise matching}} \\
&\qquad+\underbrace{\E_{X_{0:T},m_{1:T} \sim q_\theta}\left[\sum_{t=1}^{T}
\log \frac{q_{\theta,t}(m_t\mid X_t)}{p_t(m_t\mid X_{t-1})}\right]}_{\text{(4) Mask divergence}} +\; C,
\label{eq:joint_decomp}
\end{align}
where $C$ absorbs $\E_{q_\theta}[\log q(X_T)]$, constant w.r.t.\ $\theta$. We now break down each term further:

%

\textbf{Term (1): Energy}
Using $\log p_B(X_0) = -\beta H(X_0) - \log Z(\beta)$, we can isolate the energy term:
\begin{align}
-\E_{q_\theta}[\log p_B(X_0)]
&= \beta\E_{q_\theta}[H(X_0)] + \log Z(\beta). \label{eq:energy_term}
\end{align}
where $\log Z(\beta)$ can be absorbed into C as it is constant w.r.t.\ $\theta$, and $\beta=\frac{1}{\tau}$ can be removed by multiplying $\tau$ on both sides of the equation.

\textbf{Term (2): Reverse entropy / reconstruction.}
\begin{align}
&\E_{q_\theta}\left[\sum_{t=1}^{T} \log q_{\theta}(X_{t-1}\mid X_t, m_t)\right] \\
&=
\sum_{t=1}^{T}
\E_{X_{T:t-1},m_{T:t}\sim q_\theta}\left[
\log q_{\theta}(X_{t-1}\mid X_t,m_t)
\right]\\
&=
\sum_{t=1}^{T}
\E_{X_{T:t},m_{T:t}\sim q_\theta}\left[
\E_{X_{t-1}\mid X_t,m_t}
\log q_{\theta}(X_{t-1}\mid X_t,m_t)
\right]\\
&=
- \sum_{t=1}^{T}
\E_{X_{T:t},m_{T:t}\sim q_\theta}\left[
S(q_{\theta}(X_{t-1}\mid X_t,m_t))
\right].
\end{align}
Where $S(q_{\theta}(X_{t-1}\mid X_t,m_t))$ is the entropy over $q_{\theta}(X_{t-1}\mid X_t,m_t)$.

\textbf{Term (3): Forward noise matching.}
\begin{align}
&\E_{q_\theta}\left[\sum_{t=1}^{T} \log p(X_t\mid X_{t-1}, m_t)\right]\\
&=
\sum_{t=1}^{T}
\E_{X_{T:t-1},m_{T:t}\sim q_\theta}\left[
\log p(X_t\mid X_{t-1},m_t)
\right].\\
&=
\sum_{t=1}^{T}
\E_{X_{T:t},m_{T:t}\sim q_\theta}\left[
\E_{X_{t-1}\mid X_t,m_t}
\log p(X_t\mid X_{t-1},m_t)
\right].
\end{align}

\section{Invariant marginal distribution}
\label{appendix:aux_proof}

To validate our approach, we must prove that the marginal distribution of the trajectory $X_{0:T}$ under this augmented system is identical to the target diffusion process $q_\theta(X_{0:T})$ defined in Eq. (\ref{eq:joint_q_diffuco}).
Recall our proposition from \cref{sec:aux_mask}:

\textbf{Proposition 3.1} (Sampling equivalence under mask augmentation).
\textit{
Let $(X_{0:T},m_{1:T})$ be sampled from the augmented joint distribution~\eqref{eq:joint_q}. By discarding $m_{1:T}$, the resulting trajectory
$X_{0:T}$ is distributed as the original process $q_\theta(X_{0:T})$ induced by the
mask-augmented transitions~\eqref{eq:q_mask_mix}:
\(
X_{0:T} \sim q_\theta(X_{0:T})
\Longleftrightarrow
\sum_{m_{1:T}} q_\theta(X_{0:T},m_{1:T}) = q_\theta(X_{0:T}).
\)
}

\begin{proof}
We recover the marginal density of $X_{0:T}$ by summing the augmented joint distribution over all possible realizations of the auxiliary mask sequence $\mathbf{m} = \{m_1, \dots, m_T\}$:
\begin{align}
& \sum_{\mathbf{m}} q_{\theta}(X_{0:T}, \mathbf{m}) \\
& \quad = \sum_{m_1, \dots, m_T} \left[ q(X_T) \prod_{t=1}^T q_{\theta,t}(m_t \mid X_t) q_{\theta}(X_{t-1} \mid X_t, m_t) \right].
\end{align}
Because the mask $m_t$ at step $t$ is conditionally independent of masks at other steps given $X_t$, the summation over the sequence factorizes across time steps. We can interchange the summation and the product:
\begin{align}
&= q(X_T) \prod_{t=1}^T \left[ \sum_{m_t} q_{\theta,t}(m_t \mid X_t) q_{\theta}(X_{t-1} \mid X_t, m_t) \right].
\end{align}
The term inside the brackets corresponds exactly to the definition of the mask-augmented transition kernel in Eq. (\ref{eq:q_mask_mix}). Substituting this back yields:
\begin{align}
&= q(X_T) \prod_{t=1}^T q_{\theta}(X_{t-1} \mid X_t) \\
&\equiv q_\theta(X_{0:T}).
\end{align}
Therefore, discarding the auxiliary variables $m_{1:T}$ yields a sequence distributed exactly according to the original target $q_\theta(X_{0:T})$.
\end{proof}

\section{Mask annealing schedule}
\label{app:mask_schedule}

We anneal the masking rate $\rho_t\in[\rho_{\min},\rho_{\max}]$ over steps $t\in\{0,\dots,T\}$. We experiment with two schedules:

\paragraph{Linear:}
\[
\rho_t = \rho_{\min} + (\rho_{\max}-\rho_{\min})\frac{t}{T}.
\]

\paragraph{Geometric:}
\[
\rho_t = \rho_{\min}\left(\frac{\rho_{\max}}{\rho_{\min}}\right)^{t/T}.
\]

\section{Adaptive Mask Selection}
\label{app:adaptive_sampling}
We replace the state-agnostic masking distribution in \Cref{eq:mask_bern} with a \emph{state-dependent} strategy at inference time.
Concretely, instead of masking each variable with a common rate $\rho_t$, we use per-variable rates $\rho_{t}^{(i)}$ that depend on the current model output.
We normalize the rates to preserve the expected mask budget:
$
\frac{1}{d}\sum_{i=1}^d \rho_{t}^{(i)} \approx \rho_t.
$
We then sample masks independently,
\begin{equation}
\label{eq:adaptive_bernoulli}
\pi_t(m_t) = \prod_{i=1}^d \mathrm{Bernoulli}(m_{t}^{(i)}; \rho_{t}^{(i)}) .
\end{equation}

\paragraph{Top-probability margin.}



Following \citet{kim2025train}, we estimate the certainty of a variable using the
margin between its two most likely values. Let
$Z_t^{(i)}$ denote the model's output logits for variable $i$
at reverse step $t$. We first define
\(
q_\theta(x^{(i)} = j \mid X_t) = \mathrm{softmax}(Z_t^{(i)})_j .
\)
Then, let $j_{1}^{(i)}$ and $j_{2}^{(i)}$ be the most- and second-most likely values under
$q_\theta(x^{(i)} \mid X_t)$. The top-probability margin certainty score is
\begin{equation}
\label{eq:top_prob_margin}
c_t^{(i)}
=
\left|
q_\theta(x^{(i)} = j_{1}^{(i)} \mid X_t)
-
q_\theta(x^{(i)} = j_{2}^{(i)} \mid X_t)
\right| .
\end{equation}
A small value of $c_t^{(i)}$ indicates that the model is uncertain between two competing assignments for variable $i$.
We then bias mask sampling by replacing the state-agnostic rate $\rho_t$ with a state-dependent
masking rate $\rho_{t}^{(i)}(X_t)$ (dropping the explicit $X_t$ for readability). Concretely, we set
\[
\rho_{t}^{(i)} \propto -c_t^{(i)}
\]

This encourages repeatedly revisiting variables where the model has low confidence and may benefit from additional refinement.

\paragraph{Most-critical variables.}


Motivated by Large Neighborhood Search~\cite{lnsbook}, we also consider a mask-selection strategy that focuses updates on variables currently involved in constraint violations. At each step, we compute a violation score $v_i$ for every variable based on the current assignment by counting how much that variable contributes to the violated constraints. For example, in graph coloring, a vertex receives a higher score if many of its incident edges have both endpoints assigned the same color.

We then sample the mask with probabilities proportional to these violation scores, normalized to preserve the expected mask budget:
\[
\rho_{t,i} \propto v_i(X_t)
\]
Thus, variables that contribute more to the current constraint violation are more likely to be masked and resampled.

\paragraph{Related variables.}
We also consider another simple strategy inspired by Large Neighborhood Search \citep{shaw1998using}. Rather than selecting variables independently, this strategy selects groups of related variables that participate in common constraints. 

Let $\mathcal{C}=\{c_1,\ldots,c_M\}$ denote the set of constraints, and let
$\mathrm{vars}(c_k) \subseteq [d]$ be the variables appearing in constraint $c_k$.
At reverse step $t$, we sample a Bernoulli mask over constraints,
\[
m_{t,k}^{\mathrm{constr}} \sim \mathrm{Bernoulli}(\eta_t),
\qquad k \in [M],
\]
and define the variable destroy set as the union of variables appearing in the
selected constraints:
\[
S_t
=
\bigcup_{k : m_{t,k}^{\mathrm{constr}} = 1}
\mathrm{vars}(c_k).
 \quad m_{t}^{(i)} = \mathbb{I}\{i \in S_t\}.
\]

We choose the constraint-level sampling rate $\eta_t$ so that the resulting expected variable-level destruction rate is approximately $\rho_t$. 



\section{Transformer architecture}
\label{app:transformer}

We adopt the same Transformer architecture as \citet{xu2025selfsupervised}. This section summarizes the input representation and attention-based update rule used to model CSP assignments.

\subsection{Input representation}
We represent a CSP assignment as a set of $n$ tokens, one per variable. Throughout, we focus on tasks where all variables share the same discrete domain $D$ with $|D|=K$.

\paragraph{Value embeddings.}
Given a current assignment $\{x_i=v_i\}_{i=1}^n$, each variable token receives a learnable embedding $\mathbf{e}(v_i)$ for its current value $v_i\in D$.

\paragraph{Absolute positional encoding (APE).}
To encode variable identity, we add an absolute positional encoding $\mathrm{APE}(x_i)$. For variables with multi-dimensional indices $(i_1,\dots,i_k)$ (e.g., Sudoku), we concatenate positional encodings per dimension inspired by vision Transformers\cite{carion2020end,li2025tackling}:
\begin{equation}
\mathrm{APE}(x_{i_1,\dots,i_k})=\mathrm{Concat}(\mathrm{PE}(i_1),\dots,\mathrm{PE}(i_k)).
\end{equation}

\paragraph{Constraint graph as relative positional encoding (RPE).}
We construct a binary constraint graph $G=(V,E)$ where vertices correspond to variables and $(i,j)\in E$ if variables $i$ and $j$ co-occur in any constraint. We incorporate this structure via an additive bias on attention logits:
\begin{equation}
\mathrm{RPE}(i,j)=c\cdot \mathbb{I}[(i,j)\notin E], \qquad c\le 0,
\label{eq:rpe}
\end{equation}
which reduces (or masks) attention between unrelated variables. We use either a fixed $c$ or a learned scalar.

\subsection{Transformer update rule}
\label{app:arch_transformer}

\paragraph{Self-attention with constraint bias.}
At layer $\ell$, a standard multi-head self-attention block computes
\[
\mathbf{A}_{ij}=\frac{(\mathbf{W}^Q\mathbf{h}_i^{(\ell)})^\top(\mathbf{W}^K\mathbf{h}_j^{(\ell)})}{\sqrt{d}}+\mathrm{RPE}(i,j),
\qquad
\alpha_{ij}=\mathrm{softmax}_j(\mathbf{A}_{ij}),
\qquad
\mathbf{z}_i=\sum_j \alpha_{ij}\mathbf{W}^V\mathbf{h}_j^{(\ell)}.
\]
We then apply a position-wise feedforward network with standard residual connections and normalization (omitted for brevity).

\paragraph{Output.}
For each variable $i \in S$, the final-layer representation $\mathbf{h}_i^{(L)}$ is mapped to the parameters of a diagonal Gaussian distribution:
\[
\boldsymbol{\mu}_i = \mathbf{W}_{\mu}\mathbf{h}_i^{(L)} + \mathbf{b}_{\mu},
\qquad
\log \boldsymbol{\sigma}_i^{2} = \mathbf{W}_{\log\sigma^2}\mathbf{h}_i^{(L)} + \mathbf{b}_{\log\sigma^2},
\]
where $\boldsymbol{\mu}_i, \log \boldsymbol{\sigma}_i^{2} \in \mathbb{R}^{K}$. We then sample
\[
\mathbf{z}_i = \boldsymbol{\mu}_i + \boldsymbol{\sigma}_i \odot \boldsymbol{\epsilon},
\qquad
\boldsymbol{\epsilon} \sim \mathcal{N}(\mathbf{0}, \mathbf{I}),
\qquad
\boldsymbol{\sigma}_i = \exp\!\left(\tfrac{1}{2}\log \boldsymbol{\sigma}_i^{2}\right),
\]
using the reparameterization trick. The sampled latent $\mathbf{z}_i$ is decoded into the current solution, and used for the next diffusion step.



{
\color{black}
\subsection{GCN backbone ablation}
\label{app:gcn_backbone}

We construct a GCN variant of \methodname{} by replacing multi-head
attention with graph convolution. Both backbones use the constraint graph $G=(V,E)$ defined above, including self-loops. For Sudoku, neighboring variables are cells sharing a row, column, or $3\times3$ sub-grid; for Graph Coloring, MIS, and MaxCut, we use the corresponding problem graph. Omitting the shared normalization and dropout for brevity, the GCN aggregation at layer $\ell$ is
\[
\mathbf{z}_i
=
\mathbf{W}_{\mathrm{gcn}}^{(\ell)}
\sum_{j\in\mathcal{N}(i)\cup\{i\}}
\frac{\mathbf{h}_j^{(\ell)}}{\sqrt{d_i d_j}}
+
\mathbf{b}_{\mathrm{gcn}}^{(\ell)},
\]
where $\mathcal{N}(i)$ denotes the neighbors of variable $i$ and
$d_i=|\mathcal{N}(i)\cup\{i\}|$. The degree normalization controls the contribution of nodes with different numbers of neighbors. The remaining block structure and Gaussian output heads follow the Transformer architecture. Absolute positional encodings are disabled for the GCN.

\paragraph{Model size.}
We increase the GCN hidden dimension to 276 while keeping the number of layers fixed for each benchmark. The total trainable parameter counts differ by less than $0.3\%$ (Table~\ref{tab:gcn_params}). Both variants are trained from scratch using the same training instances, diffusion objective, masking schedule, optimizer, and training budget, and are evaluated under the same inference budget.

\begin{table}[h]
    \centering
    \caption{Total trainable parameter counts for the Transformer and GCN variants of \methodname{}.}
    \label{tab:gcn_params}
    \small
    \begin{tabular}{lrr}
        \toprule
        Benchmark & Transformer & GCN \\
        \midrule
        Sudoku      & 1,623,703 & 1,621,247 \\
        Coloring-5  &   927,887 &   926,547 \\
        Coloring-10 & 1,624,089 & 1,622,077 \\
        MIS         &   926,729 &   924,057 \\
        MaxCut      &   926,729 &   924,057 \\
        \bottomrule
    \end{tabular}
\end{table}

}

\section{Constraint programming formulations}
\label{app:cpform}
This section summarizes the CP formulations used to define energies for each benchmark.

\paragraph{Sudoku.}
Let $X_{i,j}\in\{1,\dots,9\}$ denote the value in cell $(i,j)$ for $i,j\in\{1,\dots,9\}$.
Sudoku constraints enforce uniqueness in every row, column, and $3\times 3$ subgrid:
\[
\textsc{AllDifferent}(X_{i,1},\dots,X_{i,9})\ \ \forall i,\qquad
\textsc{AllDifferent}(X_{1,j},\dots,X_{9,j})\ \ \forall j,
\]
and for each subgrid indexed by $(r,c)\in\{0,1,2\}^2$,
\[
\textsc{AllDifferent}\big(\{X_{3r+a,3c+b}\}_{a,b\in\{1,2,3\}}\big).
\]
Given clues are enforced by fixing the corresponding variables.

\paragraph{Graph Coloring.}
Given a graph $G=(V,E)$ and $k$ colors, assign each vertex $v\in V$ a variable $X_v\in[k]$.
The constraints require adjacent vertices to have different colors:
\[
X_u \neq X_v \qquad \forall (u,v)\in E.
\]

\paragraph{MaxCut.}
Given a graph $G=(V,E)$, introduce binary variables $X_v\in\{0,1\}$ indicating the side of the cut.
MaxCut can be viewed as a \emph{2-coloring MaxCSP}: each edge $(u,v)$ corresponds to an inequality constraint $X_u\neq X_v$, and the objective is to satisfy as many such constraints as possible:
\[
\max_X\ \sum_{(u,v)\in E} \mathbb{I}[X_u\neq X_v],
\]
which we write in minimization form for the energy:
\[
\min_X\ -\sum_{(u,v)\in E} \mathbb{I}[X_u\neq X_v].
\]

\paragraph{Maximum Independent Set.}
Given a graph $G=(V,E)$, introduce binary variables $X_v\in\{0,1\}$ indicating whether vertex $v$ is selected.
The independent set constraints are
\[
X_u + X_v \le 1 \qquad \forall (u,v)\in E,
\]
and the objective is to maximize set size:
\[
\max_X\ \sum_{v\in V} X_v,
\qquad\text{equivalently}\qquad
\min_X\ -\sum_{v\in V} X_v.
\]

\section{Training details}
\label{app:train_details}


Our models were trained on single NVIDIA H100 GPU nodes. The typical training time for a model ranges from 3 to 6 hours (wall clock). For all models, we used AdamW as the optimizer, with learning rate set to 0.0001, and train for 5000 epochs. For optimization tasks, we follow \citet{sanokowski2024a} and set $\lambda_j$ in \Cref{eq:cp_energy} to 1.01. We anneal $\tau$ in \Cref{eq:main_loss_func} across training epochs linearly from 1 towards 0. Our approach requires unrolling the reverse diffusion steps for each gradient update, this restricts the number of diffusion steps we can back propagate through. 
The DiT baselines were trained by unrolling five reverse diffusion steps per gradient update. Because \methodname{} is time-agnostic (shared parameters across steps), we use a cheaper single-step estimator: at each gradient update, we sample a block size corresponding to a reverse timestep and apply one reverse transition. At inference time, the trained transition is applied iteratively according to the problem-specific sampling budget.



\begin{table}[ht]
    \centering
    \caption{Hyperparameters for the best performing models.}
    \begin{tabular}{lccccc}
        \toprule
        & Sudoku & Graph-coloring-5 & Graph-coloring-10 & MIS & MaxCut \\
        \midrule
        Batch Size & 512 & 128 & 128 & 128 & 128\\
        Dropout & 0.2 & 0.1 & 0.1 & 0.1 & 0.1\\
        Layer Count & 7 & 4 & 7 & 4  & 4\\
        Head Count & 3 & 3 & 3 & 3 & 3 \\
        Embedding Size & 128 & 128 & 128 & 128 & 128 \\
        Mask Selection & margin & related & related & critical & margin\\
        $\rho_{\max}$ & 0.9 & 0.9 & 0.9 & 0.9 & 0.9 \\
        $\rho_{\min}$ & 0.5 & 0.5 & 0.5 & 0.3 & 0.5 \\
        \bottomrule
    \end{tabular}
    \label{tab:hyperparam}
\end{table}

\section{Extended related work}
\label{app:extended-related-work}

\paragraph{Constraint Programming solvers.}
Classical constraint programming provides expressive modeling abstractions for finite-domain variables, global constraints, and propagation-based search~\citep{cp-handbook}. Global constraints such as \textsc{AllDifferent} are especially important because they compactly encode structured combinatorial relations and admit specialized filtering algorithms~\citep{alldifferent,globalconstraints}. Modern solvers such as CP-SAT combine constraint propagation, SAT-style reasoning, integer programming techniques, and highly engineered heuristics~\citep{cpsat}. These solvers are extremely effective, but their proposal and repair mechanisms are largely hand-designed and are applied anew to each instance. 

\paragraph{Iterative methods for constraint optimization.}
Constraint optimization problems are often solved heuristically through iterative procedures that repeatedly propose, evaluate, and modify candidate assignments. A central classical example of iterative refinement is Large Neighborhood Search (LNS)~\citep{lnsbook,shaw1998using}.
LNS maintains an incumbent solution and alternates between a destroy step and a repair step. The destroy operator removes or unassigns selected components of the current solution, and the repair operator reconstructs the missing components to produce a new complete assignment. 
Recent work has explored learning-based LNS heuristics, including learned destroy and repair operators~\citep{wu2021learning,song2020general,sonnerat2021learning,huang2023searching,zhou2023learning,feng2025spl,delage2026neurosymbolic,falkner2022large,hottung2020neural,HOTTUNG2022103786,hottung2025neural}.

This view is closely aligned with our motivation: in many constraint problems, once an assignment is nearly feasible, progress often requires changing a selected subset of variables while preserving the rest. \methodname{} borrows this intuition by masking a block of variables and denoising that block conditional on the unmasked variables. 

Another line of work implements discrete sampling for solving combinatorial optimization problems. \citet{sun2023revisiting} show that modern gradient-based discrete MCMC and parallel neighborhood exploration can produce competitive training-free heuristics for several CO problems. More recently, \citet{feng2025regularized} propose Regularized Langevin Dynamics, which modifies discrete Langevin dynamics and achieves state-of-the-art results on several binary problems.

A broad line of work learns iterative neural heuristics for discrete reasoning and optimization. Many graph-based methods refine assignments through message passing and have been applied to constraint reasoning tasks~\citep{rrn,runcsp,anycsp}. \citet{yang2023learning} and \citet{xu2025selfsupervised,xu2026cpaior} implement recurrent Transformers as iterative heuristics. 
{
\color{black}
X$^2$GNN~\citep{acikalin2025learning} combines the exploration of multiple candidate solutions with stochastic GNN-based iterative refinement. 
}

{
\color{black}

\paragraph{Unsupervised methods for constraint optimization.}
A common approach to learning optimization heuristics without labeled solutions is to construct training losses directly from the problem objective and constraint violations. Such methods use continuous relaxations or gradient estimators to enable gradient-based learning, and have been developed for SAT and MaxCSP~\citep{amizadeh2019pdp,runcsp,ozolins2022goal}, graph optimization~\citep{schuetz2022combinatorial,wang2022unsupervised,karalias2020erdos}, and more general constraint satisfaction and constrained optimization problems~\citep{yang2022injecting,park2023self,tang2025learning,xu2025selfsupervised,xu2026cpaior}. Other approaches construct differentiable losses through geometric decompositions into feasible solutions~\citep{karalias2025geometric,rafiey2026self}. Test-time adaptation has also been explored to improve unsupervised neural optimization on individual problem instances~\citep{liao2026testtime}. 
Similarly, \methodname{} uses differentiable problem energies as part of a Boltzmann-based diffusion objective to learn iterative blocked resampling.

}

\paragraph{Diffusion and energy-based methods as iterative heuristics.}
Diffusion models define generative processes through gradual noising and learned denoising~\citep{sohl2015deep,ho2020denoising}. Discrete diffusion extends this idea to categorical and structured domains~\citep{austin2021structured}, while continuous relaxations such as analog bits generate discrete variables by diffusing in a continuous representation space~\citep{chen2023analog}. Diffusion Transformers replace specialized denoisers with Transformer backbones and have shown strong scaling behavior in generative modeling~\citep{peebles2023dit}.  

Recent work has adapted diffusion models to combinatorial optimization. DIFUSCO uses graph-based diffusion for graph optimization problems~\citep{sun2023difusco}, and DISCO studies efficient diffusion heuristics for large-scale combinatorial optimization~\citep{zhao2024disco}. DiffUCO formulates unsupervised neural combinatorial optimization as approximate sampling from a Boltzmann distribution, thereby avoiding the need for supervised solution labels~\citep{sanokowski2024a}; later work develops scalable discrete diffusion samplers for combinatorial optimization and statistical physics~\citep{sanokowski2025scalable}. Energy-based diffusion generators have also been used~\citep{wang2025energy,ired}. 
{
\color{black}
T2T~\citep{li2023from} combines diffusion-based solution distribution learning with gradient-guided search at test time, iteratively perturbing and reconstructing candidate solutions. Fast T2T~\citep{li2024fast} introduces optimization consistency models to reduce the number of denoising steps while retaining objective-guided search. In contrast, \methodname{} performs Gibbs-style blocked conditional resampling, updating selected variables while preserving the rest of the assignment, learning the reverse transitions through an unsupervised Boltzmann objective.
}


\paragraph{Masked, blocked, and Gibbs-style generation.}
Classical Gibbs sampling updates variables, or blocks of variables, conditional on the rest of the current state~\citep{geman1984stochastic,gelfand1990sampling,jensen1995blocking}. This is natural for constraint solving because the unmodified part of a nearly feasible solution provides useful context for repairing the violated part.

Several recent generative models also use partial updates. SDEdit shows that re-noising and re-denoising can perform controlled edits in continuous domains~\citep{meng2022sdedit}. Masked diffusion language models generate discrete sequences through iterative masking and denoising~\citep{sahoo2024simple,shi2024simplified}, and remasking improves inference-time scaling by deciding which variables should be revised~\citep{wang2025remasking}. Work on token ordering in masked diffusion further shows that the choice of variables to reveal or update can strongly affect generation quality~\citep{kim2025train}. These ideas are conceptually aligned with \methodname{}'s adaptive mask selection, where variables are prioritized according to model uncertainty or constraint violation.

Two recent methods are particularly close and are useful points of contrast. Interleaved Gibbs Diffusion performs Gibbs-style conditional updates for constrained generation~\citep{anil2025interleaved}, but updates one variable at a time and uses supervised training. \methodname{} instead updates a block at each iteration and trains without labels using a Boltzmann objective. Block Diffusion is also related, but it is a language modeling method whose blocks are fixed sequential segments of a token sequence; it runs a complete diffusion process within each block before moving to the next~\citep{arriola2025block}. In contrast, \methodname{} samples a different block at each diffusion step, anneals the block size over time, and uses each block update as a learned repair move for constraint optimization.

\section{Baseline methods}
\label{app:baselines}

Our baseline comparisons build on results reported in prior work on the same benchmarks, supplemented by additional evaluations. We include both classical optimization methods and learned heuristics to provide context for solution quality and computational cost. Below, we briefly describe the baselines and identify the sources of their reported results.


\paragraph{Sudoku.}
We compare with SATNet~\cite{satnet}, a differentiable MAXSAT solver based on solving an SDP relaxation; RRN~\cite{rrn}, a recurrent message-passing model for multi-step relational reasoning; the Recurrent Transformer~\cite{yang2023learning}, which extends Transformers with recurrence for end-to-end CSP solving; iRED~\cite{ired}, which formulates reasoning as energy-based optimization and performs iterative inference via energy diffusion; RLSA~\cite{feng2025regularized}, where we formulate Sudoku as a QUBO following \cref{sec:sudoku_qubo}; and ConsFormer\cite{xu2025selfsupervised}.

\paragraph{Graph Coloring.}
We compare with the greedy coloring algorithm implemented by NetworkX~\cite{networkx}; Feasibility jump, which is a local search heuristic~\cite{luteberget2023feasibility}; ANYCSP~\cite{anycsp}, which proposes a universal GNN trained to act as an end-to-end heuristic for CSPs; RLSA~\cite{feng2025regularized}, where we formulate graph coloring as a QUBO following \cref{sec:graph_coloring_qubo}; and ConsFormer\cite{xu2025selfsupervised}.

\paragraph{MaxCut.}
We compare with RUNCSP~\cite{runcsp}, an unsupervised GNN approach for maximum constraint satisfaction over binary CSPs; ECO-DQN~\cite{barrett2020exploratory}, an RL method that improves solutions by exploratory local modifications during inference; ECORD~\cite{ecord}, an RL approach that amortizes expensive GNN computation; X$^2$GNN~\cite{acikalin2025learning}, an unsupervised GNN-based optimization method that combines exploration of candidate solutions with stochastic iterative refinement; RLSA~\cite{feng2025regularized}, ANYCSP~\cite{anycsp}, DiffUCO~\cite{sanokowski2024a}, and ConsFormer\cite{xu2025selfsupervised}. 

\paragraph{MIS.}
Following \citet{sanokowski2024a}, we compare with INTEL~\cite{li2018combinatorial} and DGL~\cite{bother2022whats}, which use a Graph Convolutional Network to predict node membership likelihoods and then performs guided tree search; LwD~\cite{ahn2020learning}, which learns a policy for deferring decisions to construct MIS in a variable number of stages; LTFT~\cite{zhang2023let}, which trains conditional GFlowNets to sequentially construct solutions and sample from an unnormalized objective over the solution space; DIFUSCO~\cite{sun2023difusco}, a supervised graph-based diffusion solver that learns to generate combinatorial solutions from reference solutions; RLSA~\cite{feng2025regularized}, X$^2$GNN~\cite{acikalin2025learning}, and DiffUCO~\cite{sanokowski2024a}.


\begin{table}[h]
\centering
\caption{Sources of baseline results. ``New'' denotes numbers obtained by running the method under the evaluation protocol used in this paper. Otherwise, numbers are taken from the cited source.}
\label{tab:baseline_sources}
\resizebox{\linewidth}{!}{
\begin{tabular}{lll}
\toprule
Benchmark & Methods & Source \\
\midrule
Sudoku
& OR-Tools, SATNet, RRN, Recurrent Transformer, IRED, ConsFormer
& \citet{xu2025selfsupervised} \\
&  RLSA 
& New \\
\midrule
Graph Coloring
& OR-Tools, Greedy, Feasibility Jump, ANYCSP, ConsFormer
& \citet{xu2025selfsupervised} \\
& RLSA 
& New \\
\midrule
MIS
& LwD, INTEL, DGL, LTFT, DiffUCO
& \citet{sanokowski2024a} \\
& DIFUSCO, RLSA
& \citet{feng2025regularized} \\
& X$^2$GNN 
& \citet{acikalin2025learning} \\
& OR-Tools
& New \\
\midrule
MaxCut
& RUNCSP, ECO-DQN, ECORD, DiffUCO
& \citet{sanokowski2024a} \\
& OR-Tools, ANYCSP, ConsFormer
& \citet{xu2025selfsupervised} \\
& X$^2$GNN 
& \citet{acikalin2025learning} \\
& RLSA
& New \\
\bottomrule
\end{tabular}}
\end{table}

\subsection{Sudoku QUBO formulation}
\label{sec:sudoku_qubo}

Following~\citet{mucke2024simple}, we represent a Sudoku assignment with binary variables \(X_{r,c,d} \in \{0,1\}\), where \(X_{r,c,d}=1\) indicates that digit \(d\) is placed in cell \((r,c)\). 
The QUBO energy is 
\[ E(X) = -\sum_{r,c,d} X_{r,c,d} + \lambda
\sum_{\{(r,c,d),(r',c',d')\}\in \mathcal{C}}
X_{r,c,d}X_{r',c',d'},
\]
where \(\mathcal{C}\) is the set of conflicting assignment pairs. Two assignments conflict if they place the same digit in the same row, column, or \(3\times3\) box, or if they assign different digits to the same cell. We use \(\lambda=3\), as in \citet{mucke2024simple}.
Given clues are incorporated by clamping the corresponding variables: if cell \((r,c)\) is given as digit \(d\), then \(X_{r,c,d}=1\) and \(X_{r,c,d'}=0\) for all \(d'\neq d\).

\subsection{Graph Coloring QUBO formulation}
\label{sec:graph_coloring_qubo}
Following \citet{glover2018tutorial}, let \(G=(V,E)\) be an undirected graph and let \(K\) be the number of available colors. We introduce binary variables \(X_{v,k}\in\{0,1\}\), where \(X_{v,k}=1\) indicates that vertex \(v\) is assigned color \(k\). The QUBO energy is
\[ E(X) = \sum_{v\in V} \left(1 - \sum_{k=1}^{K} X_{v,k} \right)^2 + \sum_{(u,v)\in E} \sum_{k=1}^{K} X_{u,k}X_{v,k}.
\]
The first term enforces that each vertex receives exactly one color, while the second penalizes adjacent vertices assigned the same color.

\section{Computational efficiency}
\label{app:efficiency}

No prior work evaluates all four benchmarks under a single unified protocol, and different benchmark families follow different reporting conventions. We therefore follow the standard evaluation convention for each benchmark and report the corresponding wall-clock budgets and sampling steps in \Cref{tab:compute_cost}. For Sudoku, we use a fixed budget of $2000$ reverse diffusion steps following~\citet{xu2025selfsupervised}. For Graph Coloring, all methods are evaluated under a strict $10$-second wall-clock limit per instance, following~\citet{xu2025selfsupervised}. For MaxCut, we use a $180$-second wall-clock limit and report the best result over $20$ independent runs, following~\citet{sanokowski2024a}. For MIS, comparison is less standardized across prior work, so we use a fixed budget of $1000$ reverse diffusion steps and report the measured runtime, for reference, OR-Tools results averaged 0.39 seconds per instance for RB-small, and 18.43 seconds for RB-large.

\begin{table}[!h]
    \centering
    \begin{tabular}{l|c|c}
         \toprule
         Dataset & Runtime (Seconds) & Inference Steps\\
         \midrule
         Sudoku-SATNet & 0.15 & 2000 \\
         Sudoku-RRN & 0.15 & 2000\\
         Graph-Coloring-5 ($n=50$)  & 10 & 5399\\
         Graph-Coloring-5 ($n=100$)  & 10 & 5439 \\
         Graph-Coloring-10 ($n=100$) & 10 & 3728 \\
         Graph-Coloring-10 ($n=200$)  & 10 & 3931 \\
         MaxCut-800 & 180 & 84853\\
         MaxCut-1K & 180 & 75963\\
         MaxCut-2K   & 180 & 23569 \\
         MaxCut-$\ge$3K  & 180 & 6840 \\
         MIS-RB-small & 0.44 & 1000 \\
         MIS-RB-large & 4.81 & 1000 \\
         \bottomrule
    \end{tabular}
    \caption{Computation cost at inference time. Sudoku and MIS use fixed sampling-step budgets, and the table reports the resulting average runtime per instance. Graph Coloring and MaxCut use fixed wall-clock budgets, and the table reports the average number of sampling steps completed within that budget.}
    \label{tab:compute_cost}
\end{table}

{
\color{black}

\subsection{MaxCut GSET computation protocols}
All methods except X$^2$GNN use a 180-second budget and report the best of 20 runs following~\citet{sanokowski2024a}. For X$^2$GNN, using the values reported in \citet{acikalin2025learning}, we compute the average absolute gaps to the same best-known reference values used for our results. X$^2$GNN uses different training distributions, inference budgets, and restart settings.

\subsection{Sudoku and MIS runtime comparisons}
\label{app:runtime_comparison}

We additionally report wall-clock runtimes for Sudoku and MIS. Table~\ref{tab:mis_runtime} compares the average solution quality and runtime per instance on the RB benchmarks. The runtimes of learned baselines are taken from the corresponding literature. As these methods may use different hardware and implementations, the reported runtimes are intended to provide additional context rather than a strictly equal hardware comparison.
Table~\ref{tab:sudoku_runtime} compares \methodname{} with OR-Tools and RLSA on Sudoku. OR-Tools achieves the best performance in both accuracy and runtime. In contrast, \methodname{} substantially outperforms RLSA in accuracy while requiring less inference time. 

\begin{table}[h]
    \centering
    \caption{Runtime comparison on MIS. We report the average independent set size ($\uparrow$) and wall-clock time per instance in seconds ($\downarrow$).}
    \label{tab:mis_runtime}
    \small
    \setlength{\tabcolsep}{6pt}
    \begin{tabular}{lrrrr}
        \toprule
        & \multicolumn{2}{c}{RB-small}
        & \multicolumn{2}{c}{RB-large} \\
        \cmidrule(lr){2-3}
        \cmidrule(lr){4-5}
        Method & Size $\uparrow$ & Time $\downarrow$
               & Size $\uparrow$ & Time $\downarrow$ \\
        \midrule
        OR-Tools    & 20.10 & 0.39 & 42.66 & 18.43 \\
        RLSA        & 19.97 & 0.07 & 40.19 & 0.22 \\
        INTEL       & 18.47 & 1.57 & 34.47 & 2.43 \\
        DGL         & 17.36 & 1.53 & 34.50 & 2.87 \\
        LTFT        & 19.18 & 0.06 & 37.48 & 0.52 \\
        DIFUSCO     & 18.52 & 1.93 & --    & --   \\
        X$^2$GNN    & 20.07 & 1.94 & 39.28 & 9.43 \\
        DiffUCO     & 18.88 & 0.11 & 38.10 & 0.59 \\
        \methodname{} & 19.52 & 0.44 & 37.05 & 4.81 \\
        \bottomrule
    \end{tabular}
\end{table}

\begin{table}[h]
    \centering
    \caption{Runtime comparison on Sudoku. We report the percentage of solved instances ($\uparrow$) and average wall-clock time per instance in seconds ($\downarrow$).}
    \label{tab:sudoku_runtime}
    \small
    \begin{tabular}{lrrrr}
        \toprule
        & \multicolumn{2}{c}{SATNet}
        & \multicolumn{2}{c}{RRN} \\
        \cmidrule(lr){2-3}
        \cmidrule(lr){4-5}
        Method & Solved $\uparrow$ & Time $\downarrow$
               & Solved $\uparrow$ & Time $\downarrow$ \\
        \midrule
        OR-Tools     & 100.0 & 0.004 & 100.0 & 0.004 \\
        RLSA         & 4.2   & 0.38  & 1.8   & 0.38 \\
        \methodname{} & 100.0 & 0.15  & 94.1  & 0.15 \\
        \bottomrule
    \end{tabular}
\end{table}

}

\section{Details for obtaining 2D projected trajectory visualization}
\label{app:traj_plot}

To visualize the solving process, we project trajectories into a shared 2D subspace defined only over the \emph{unfixed Sudoku cells}. For a puzzle with mask \(m \in \{0,1\}^{81}\), where \(m_i=1\) indicates a given cell, each solution state is represented as a matrix \(X \in [0,1]^{81 \times 9}\) of per-cell categorical probabilities. We discard the fixed cells and flatten the remaining \(U\) unfixed rows into a vector in \(\mathbb{R}^{9U}\). To provide context for the projection, we construct a static background cloud of \(N\) reduced states consisting of: (i) random assignments on the unfixed cells, and (ii) partially corrupted versions of the target solution, where a random fraction of unfixed cells is replaced by random digits. We then fit a single 2D dimensionality reduction map \(f\) on this static set using PCA, and project all trajectories with the same fitted map. Trajectory points are colored by accuracy, computed as the number of cells whose decoded digit matches the target solution. The resulting figure shows a hexbin/KDE background over the projected static cloud, together with the projected paths and the noise/target anchors. This visualization illustrates broad differences in trajectory behavior, but the 2D geometry and background density should not be interpreted as the true energy landscape.

{
\color{black}

\section{Additional MIS Baselines}
\label{app:additional_mis_baselines}

We additionally compare \methodname{} with T2T, Fast T2T, and X$^2$GNN on the ER[700--800] MIS benchmark. These methods also improve solutions through iterative perturbation or reconstruction, but differ in their learning objectives and architectures. 

Table~\ref{tab:mis_er_baselines} reports the average independent set size and runtime per instance. X$^2$GNN achieves the largest independent set size, while Fast T2T has the lowest runtime. \methodname{} outperforms T2T, Fast T2T, and OR-Tools in solution quality, but requires more inference time than T2T and Fast T2T. Compared with X$^2$GNN, \methodname{} achieves a smaller independent set size at a lower runtime. These results illustrate the trade-off between solution quality and computational cost across the different approaches.

\begin{table}[h]
    \centering
    \caption{Additional MIS baseline comparison on
    ER[700--800]. We report the average independent
    set size ($\uparrow$) and runtime per instance
    in seconds ($\downarrow$).}
    \label{tab:mis_er_baselines}
    \small
    \begin{tabular}{lrr}
        \toprule
        Method & Size $\uparrow$ & Time (s) $\downarrow$ \\
        \midrule
        T2T           & 41.37 & 13.92 \\
        Fast T2T      & 41.73 & \textbf{2.74} \\
        X$^2$GNN      & \textbf{44.91} & 57.18 \\
        OR-Tools      & 41.97 & 120.00 \\
        \methodname{} & 42.20 & 32.90 \\
        \bottomrule
    \end{tabular}
\end{table}
}

\section{MaxCut per instance performance}
\label{app:maxcut_perins}

The convention for MaxCut is to conduct 20 parallel runs and report the best cut found~\cite{ecord,anycsp,xu2025selfsupervised}. In \cref{tab:per_instance_cut_comparison}, we additionally detail the mean and standard deviation of the 20 runs for each instance for ablation. 
We see that \methodname{} obtains a higher mean cut value than the $-$AMS and $-$Entropy variants on 19 and 21 of the 35 instances, respectively. The advantage is most notable on the largest instances: \methodname{} outperforms both ablations on all instances with at least 5,000 vertices. However, the improvements are not uniform across instances, and the complete model does not consistently reduce run-to-run variability.

\begin{table}[h]
\centering
\caption{
Per-instance MaxCut values reported as mean $\pm$ standard deviation across runs for each ablation. The highest mean among the three variants is shown in bold.
}
\label{tab:per_instance_cut_comparison}
\resizebox{0.9\linewidth}{!}{
\begin{tabular}{ccllll}
\toprule
$|V|$ & ID & Best Known
& \methodname{} & $-$AMS & $-$Entropy \\
\midrule
800 & G1  & 11624
& $11609.88 \pm 17.34$
& $\mathbf{11619.95 \pm 6.95}$
& $11613.15 \pm 13.28$ \\

800 & G14 & 3064
& $\mathbf{3060.12 \pm 2.15}$
& $3057.65 \pm 1.76$
& $3057.45 \pm 1.54$ \\

800 & G15 & 3050
& $\mathbf{3048.82 \pm 0.95}$
& $3046.40 \pm 1.19$
& $3046.10 \pm 3.96$ \\

800 & G16 & 3052
& $\mathbf{3049.00 \pm 2.24}$
& $3044.60 \pm 2.60$
& $3044.45 \pm 2.28$ \\

800 & G17 & 3047
& $\mathbf{3042.65 \pm 2.21}$
& $3039.60 \pm 1.88$
& $3040.15 \pm 2.37$ \\

800 & G2  & 11620
& $11606.82 \pm 9.61$
& $11609.20 \pm 7.54$
& $\mathbf{11609.90 \pm 9.47}$ \\

800 & G3  & 11622
& $11614.47 \pm 8.18$
& $\mathbf{11617.55 \pm 3.72}$
& $11616.95 \pm 7.14$ \\

800 & G4  & 11646
& $11630.65 \pm 10.17$
& $11640.80 \pm 4.51$
& $\mathbf{11641.30 \pm 6.25}$ \\

800 & G5  & 11631
& $11621.71 \pm 11.26$
& $\mathbf{11626.85 \pm 1.66}$
& $11622.20 \pm 5.43$ \\
\midrule
1000 & G43 & 6660
& $6650.18 \pm 10.38$
& $\mathbf{6657.90 \pm 2.88}$
& $6653.80 \pm 6.99$ \\

1000 & G44 & 6650
& $6644.12 \pm 8.34$
& $\mathbf{6647.30 \pm 1.72}$
& $6646.50 \pm 5.78$ \\

1000 & G45 & 6654
& $6640.06 \pm 6.56$
& $\mathbf{6648.85 \pm 6.34}$
& $6645.90 \pm 4.47$ \\

1000 & G46 & 6649
& $6638.06 \pm 5.49$
& $\mathbf{6644.10 \pm 3.73}$
& $6641.85 \pm 5.33$ \\

1000 & G47 & 6657
& $6647.24 \pm 5.61$
& $\mathbf{6653.15 \pm 4.25}$
& $6648.35 \pm 6.78$ \\

1000 & G51 & 3848
& $\mathbf{3843.06 \pm 2.44}$
& $3838.20 \pm 2.04$
& $3839.75 \pm 1.41$ \\

1000 & G52 & 3851
& $\mathbf{3845.65 \pm 2.80}$
& $3839.30 \pm 2.13$
& $3839.95 \pm 2.26$ \\

1000 & G53 & 3850
& $\mathbf{3845.29 \pm 2.57}$
& $3839.55 \pm 1.70$
& $3841.15 \pm 1.93$ \\

1000 & G54 & 3852
& $\mathbf{3844.76 \pm 3.23}$
& $3841.05 \pm 2.50$
& $3842.60 \pm 3.27$ \\
\midrule
2000 & G22 & 13359
& $13318.59 \pm 19.17$
& $\mathbf{13340.75 \pm 15.12}$
& $13333.10 \pm 17.48$ \\

2000 & G23 & 13344
& $13312.53 \pm 13.64$
& $13322.45 \pm 6.46$
& $\mathbf{13325.50 \pm 8.82}$ \\

2000 & G24 & 13337
& $13313.82 \pm 9.52$
& $\mathbf{13320.70 \pm 6.41}$
& $13316.50 \pm 8.78$ \\

2000 & G25 & 13340
& $13316.71 \pm 8.32$
& $\mathbf{13317.05 \pm 5.65}$
& $13316.60 \pm 9.08$ \\

2000 & G26 & 13328
& $13303.47 \pm 11.14$
& $13307.00 \pm 9.27$
& $\mathbf{13307.60 \pm 8.19}$ \\

2000 & G35 & 7687
& $\mathbf{7667.76 \pm 4.74}$
& $7646.10 \pm 4.20$
& $7651.15 \pm 5.90$ \\

2000 & G36 & 7680
& $\mathbf{7659.29 \pm 3.93}$
& $7638.35 \pm 4.85$
& $7647.20 \pm 2.33$ \\

2000 & G37 & 7691
& $\mathbf{7670.59 \pm 4.93}$
& $7652.45 \pm 5.58$
& $7661.30 \pm 7.53$ \\

2000 & G38 & 7688
& $\mathbf{7668.12 \pm 3.67}$
& $7649.25 \pm 7.20$
& $7651.60 \pm 5.06$ \\
\midrule
3000 & G48 & 6000
& $\mathbf{5969.18 \pm 49.46}$
& $5945.50 \pm 56.26$
& $5925.20 \pm 46.97$ \\

3000 & G49 & 6000
& $\mathbf{5954.00 \pm 40.09}$
& $5942.50 \pm 31.52$
& $5932.50 \pm 30.23$ \\

3000 & G50 & 5880
& $5852.71 \pm 24.83$
& $\mathbf{5853.60 \pm 24.22}$
& $5841.50 \pm 17.48$ \\
5000 & G55 & 10299
& $\mathbf{10264.06 \pm 11.13}$
& $10258.55 \pm 7.88$
& $10256.05 \pm 10.23$ \\

5000 & G58 & 19293
& $\mathbf{19223.18 \pm 5.50}$
& $19145.15 \pm 12.88$
& $19151.55 \pm 12.22$ \\
7000 & G60 & 14188
& $\mathbf{14138.06 \pm 7.45}$
& $14081.05 \pm 11.24$
& $14118.55 \pm 8.65$ \\

7000 & G63 & 27045
& $\mathbf{26936.06 \pm 12.48}$
& $26657.45 \pm 20.50$
& $26825.40 \pm 12.12$ \\
10000 & G70 & 9591
& $\mathbf{9508.18 \pm 8.65}$
& $9467.95 \pm 7.49$
& $9491.35 \pm 7.81$ \\
\bottomrule
\end{tabular}
}
\end{table}

\section{Additional ablations}
\label{app:ablate}

{
\color{black}

\subsection{Inference budget sensitivity}
\label{app:ablate-T}

We further evaluate \methodname{} using $0.5T$, $T$, and $2T$ inference steps, where $T$ denotes the reference sampling budget for each benchmark (Appendix~\ref{app:efficiency}). As shown in Table~\ref{tab:step_ablation}, Graph Coloring and MIS generally benefit from additional steps, although the improvements are modest. MaxCut exhibits a less consistent trend: increasing the number of steps sometimes degrades performance. One possible explanation is that extending the annealing schedule introduces additional large-block updates, which may disrupt high-quality assignments. These results suggest that increasing the inference budget does not necessarily improve solution quality.

\begin{table}[h]
    \centering
    \caption{Effect of the number of inference steps.
    We evaluate using half ($0.5T$) and twice ($2T$)
    the reference number of steps ($T$).
    Higher is better for Graph Coloring (percentage solved)
    and MIS (average independent set size);
    lower is better for MaxCut (gap to best-known solution).}
    \label{tab:step_ablation}
    \small
    \begin{tabular}{lrrr}
        \toprule
        Dataset & $0.5T$ & $T$ & $2T$ \\
        \midrule
        Coloring-5 ($n=50$)  & 82.58 & 83.00 & 83.00 \\
        Coloring-5 ($n=100$) & 52.33 & 54.67 & 54.92 \\
        Coloring-10 ($n=100$)& 53.92 & 53.92 & 53.92 \\
        Coloring-10 ($n=200$)& 17.92 & 18.58 & 18.83 \\
        MIS-RB-small         & 19.03 & 19.52 & 19.53 \\
        MIS-RB-large         & 36.49 & 37.05 & 37.51 \\
        MIS-ER               & 40.74 & 42.20 & 42.59 \\
        MaxCut-800           & 0.22  & 0.11  & 0.56 \\
        MaxCut-1K            & 0.44  & 1.33  & 0.56 \\
        MaxCut-2K            & 5.55  & 7.88  & 6.11 \\
        MaxCut-$\geq$3K      & 44.13 & 32.75 & 23.75 \\
        \bottomrule
    \end{tabular}
\end{table}

{
\color{black}
\subsection{Adaptive mask-selection}
\label{app:mask_sensitivity}

We study the effect of the mask-selection strategy by comparing random selection with the three adaptive strategies described in Appendix~\ref{app:adaptive_sampling}.
Table~\ref{tab:mask_sensitivity} reports the results across all four problems. The choice of mask-selection strategy can have a substantial effect on performance. Confidence-based selection achieves the highest accuracy on Sudoku-Hard, while selecting variables involved in constraint violations performs best on MIS. On Graph Coloring, related-variable selection generally performs best, although the differences are relatively small. For MaxCut, both confidence-based and critical-variable selection achieve substantially smaller gaps than random masking.
These results suggest that adaptive mask selection benefits from problem-specific information, while strategies that are poorly matched to the problem may perform worse than random selection.

\begin{table}[!h]
    \centering
    \caption{Sensitivity to the mask-selection strategy.
    We compare random selection with related-variable,
    most-critical-variable, and confidence-margin
    selection. Higher is better for Sudoku and Graph
    Coloring (\% solved) and MIS (average independent
    set size); lower is better for MaxCut (gap to
    best-known solution).}
    \label{tab:mask_sensitivity}
    \resizebox{0.67\linewidth}{!}{
    \setlength{\tabcolsep}{5pt}
    \begin{tabular}{lrrrr}
        \toprule
        Dataset & Random & Related & Critical & Confidence \\
        \midrule
        Sudoku-Easy                   & 100.00 & 100.00 & 100.00 & 100.00 \\
        Sudoku-Hard                   & 88.21  & 80.53  & 78.27  & 94.12 \\
        Coloring-5 ($n=50$)           & 82.92  & 83.00  & 82.33  & 82.25 \\
        Coloring-5 ($n=100$)          & 54.58  & 54.67  & 50.33  & 51.42 \\
        Coloring-10 ($n=100$)         & 53.92  & 53.92  & 53.92  & 53.92 \\
        Coloring-10 ($n=200$)         & 17.42  & 18.58  & 16.92  & 16.83 \\
        MIS-RB-small (size $\uparrow$) & 15.83  & 16.22  & 19.52  & 18.46 \\
        MIS-RB-large (size $\uparrow$) & 17.16  & 20.99  & 37.05  & 34.08 \\
        MaxCut-800 (gap $\downarrow$)  & 2.11   & 11.11  & 2.22   & 0.11 \\
        MaxCut-1K (gap $\downarrow$)   & 3.22   & 11.56  & 1.22   & 1.33 \\
        MaxCut-2K (gap $\downarrow$)   & 18.89  & 38.89  & 7.11   & 7.88 \\
        MaxCut-$\geq$3K (gap $\downarrow$)
                                      & 88.63  & 117.25 & 34.13  & 32.75 \\
        \bottomrule
    \end{tabular}}
\end{table}
}

\subsection{Masking rate and realized changes on Sudoku}
\label{app:realized_change}

We extend the masking rate ablation in \Cref{sec:ablate} to Sudoku, varying $\rho_{\max}$ and $\rho_{\min}$ while following the same geometric annealing schedule. We additionally measure the percentage of masked variables whose discrete assignments change after each reverse diffusion step. This allows us to distinguish the masking rate from the actual number of variables modified by the model.

Table~\ref{tab:sudoku_masking} reports accuracy on Sudoku-RRN and the realized change rates at selected inference steps. Similar to MIS, schedules that begin with large blocks and anneal to moderately smaller blocks achieve the best performance. The realized change rate also decreases substantially over inference, indicating that the model increasingly preserves assignments even among the selected variables. In particular, for the best-performing schedule $(\rho_{\max},\rho_{\min})=(0.9,0.5)$, the percentage of masked variables that change decreases from 88.86\% at the first step to 1.00\% at the final step.

These results suggest that relatively large masking rates do not necessarily lead to equally large changes in the assignment. At the same time, excessively small final masking rates can restrict the variables available for repair, whereas large final masking rates may disrupt high-quality assignments.

\begin{table}[h]
    \centering
    \caption{Masking-rate ablation on Sudoku-RRN.
    We report solution accuracy (\%) and the percentage of
    masked variables whose discrete assignments change
    at selected inference steps. The masking rate is
    annealed geometrically from $\rho_{\max}$ to
    $\rho_{\min}$.}
    \label{tab:sudoku_masking}
    \resizebox{0.8\linewidth}{!}{
    \setlength{\tabcolsep}{4pt}
    \begin{tabular}{@{}rrc|rrrrr@{}}
        \toprule
        $\rho_{\max}$ & $\rho_{\min}$ & Accuracy $\%$
        & Step 1 & Step 2 & Step 100
        & Step 1000 & Step 2000 \\
        \midrule
        0.3 & 0.1 & 30.10 & 88.11 & 62.86 & 3.82  & 0.83 & 0.48 \\
        0.5 & 0.1 & 57.18 & 87.92 & 52.61 & 3.74  & 0.65 & 0.23 \\
        0.5 & 0.3 & 67.52 & 88.10 & 52.68 & 3.46  & 1.07 & 0.62 \\
        0.7 & 0.1 & 81.79 & 88.29 & 32.71 & 3.51  & 0.41 & 0.14 \\
        0.7 & 0.3 & 88.58 & 88.59 & 33.43 & 5.73  & 0.71 & 0.33 \\
        0.7 & 0.5 & 93.36 & 88.60 & 33.99 & 9.60  & 3.77 & 1.27 \\
        0.9 & 0.1 & 87.81 & 85.09 & 28.18 & 14.09 & 0.39 & 0.12 \\
        0.9 & 0.3 & 91.73 & 86.40 & 25.65 & 17.52 & 1.50 & 0.35 \\
        0.9 & 0.5 & \textbf{94.12}
                  & 88.86 & 14.47 & 6.20 & 7.39 & 1.00 \\
        0.9 & 0.7 & 83.54 & 88.49 & 15.44 & 5.76  & 6.37 & 5.32 \\
        \bottomrule
    \end{tabular}
    }
\end{table}
}

\end{document}